%% file: main.tex
\documentclass[journal]{IEEEtran}
\ifCLASSINFOpdf
\else
\usepackage[dvips]{graphicx}
\fi

\usepackage{cite}
\usepackage{booktabs}
\usepackage[pdftex]{graphicx}
\usepackage[cmex10]{amsmath}
\usepackage{multirow}
\usepackage{color}
\usepackage{mathtools}
\usepackage{array}
\usepackage{dblfloatfix}
\usepackage{subfig}
\usepackage{bm}
\input{math}

\usepackage{balance}
\usepackage{url}
\usepackage{multirow, float, booktabs, boldline, hhline}
\usepackage{cellspace}
\usepackage{color, colortbl}
\usepackage{algorithm}
\usepackage{algpseudocode}
\usepackage{xcolor}
\usepackage{amsmath,amssymb}
\usepackage{amsthm}
\usepackage{soul}
\newtheorem{theorem}{Theorem}
\newtheorem{lemma}[theorem]{Lemma}

\usepackage{cite}

\newcommand{\eg}{\emph{e.g.}}
\begin{document}
\title{Distributed Zero-Shot Learning for\\Visual Recognition}


\author{Zhi~Chen,  Yadan~Luo, Zi~Huang, Jingjing~Li, Sen~Wang,  Xin~Yu
\thanks{Z. Chen is with the University of Southern Queensland, Toowoomba, QLD 4350, Australia. (e-mail:uqzhichen@gmail.com)}
\thanks{Y. Luo, Z. Huang,  S. Wang and X. Yu, are with School of Electrical Engineering \& and Computer Science, The University of Queensland, Brisbane, QLD 4072, Australia. (e-mails:lyadanluol@gmail.com, huang@itee.uq.edu.au, sen.wang@uq.edu.au, yu.xin@uq.edu.au).}
\thanks{J. Li is with the School of Computer Science and Engineering, University of Electronic Science and Technology of China (email: jjl@uestc.edu.cn).}
\thanks{This paper has supplementary downloadable material available at http://ieeexplore.ieee.org., provided by the authors. The material includes a comprehensive theoretical analysis and experimental results. This material is 6 pages.}}

\maketitle

\begin{abstract}
In this paper, we propose a Distributed Zero-Shot Learning (DistZSL) framework that can fully exploit decentralized data to learn an effective model for unseen classes.
Considering the data heterogeneity issues across distributed nodes, we introduce two key components to ensure the effective learning of DistZSL: 
a cross-node attribute regularizer and a global attribute-to-visual consensus.
Our proposed cross-node attribute regularizer enforces the distances between attribute features to be similar across different nodes. In this manner, the overall attribute feature space would be stable during learning, and thus facilitate the establishment of visual-to-attribute (V2A) relationships.
Then, we introduce the global attribute-to-visual consensus to mitigate biased V2A mappings learned from individual nodes. Specifically, we enforce the bilateral mapping between the attribute and visual feature distributions to be consistent across different nodes. 
Thus, the learned consistent V2A mapping can significantly enhance zero-shot learning across different nodes. 
Extensive experiments demonstrate that DistZSL achieves superior performance to the state-of-the-art in learning from distributed data.
\end{abstract}

\IEEEpeerreviewmaketitle
\section{Introduction}
\IEEEPARstart{V}isual recognition aims to identify and categorize visual data, forming a cornerstone in the field of computer vision.
With the ever-growing amount of data, the ability to recognize instances from previously seen and unseen classes is highly desired. To this end, Generalized Zero-Shot Learning (GZSL) \cite{xian2018feature,liu2018generalized,su2022distinguishing,schonfeld2019generalized,guo2024element} has been provided. 
GZSL approaches usually first establish a visual-attribute mapping and then exploit it to classify seen classes (\textit{i.e.}, available in both training and test) and unseen classes (\textit{i.e.}, only appear in testing) based on their attribute descriptions.
Current GZSL methods often require a large amount of centralized data in training. However, when data cannot be shared or centralized, previous methods might fail to achieve satisfactory performance.

\begin{figure}
    \centering
    \includegraphics[width=1\linewidth]{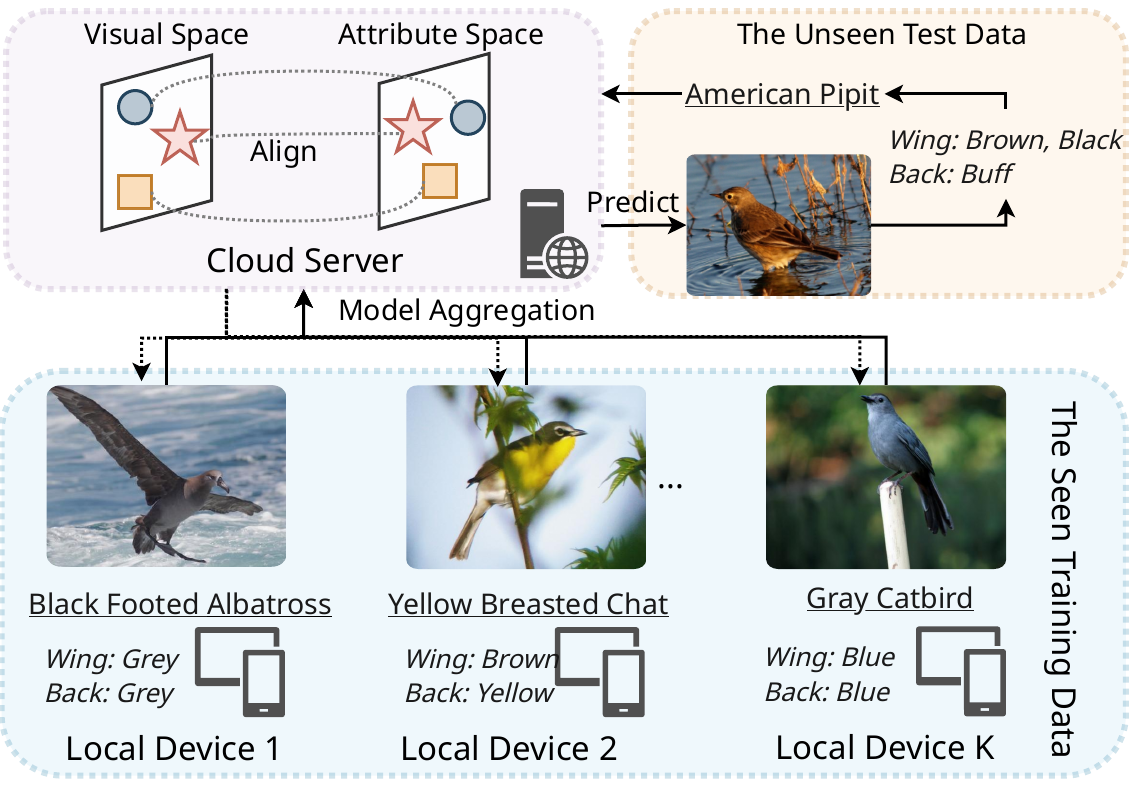}
    \caption{An illustration of Distributed Zero-Shot Learning ({DistZSL}), which aims to infuse ZSL capability into distributed learning frameworks.}
    \label{fig:intro}
    \vspace{-10pt}
\end{figure}

Federated Learning (FL) is considered an appealing distributed learning framework, as it only exchanges model parameters instead of original data.
Direct incorporation of existing GZSL methods into FL may not lead to adequate performance. This is because models trained on individual nodes would yield different visual-attribute mappings, and the inconsistent mappings do not help zero-shot classification. These mappings will become even more inconsistent when the data are distributed heterogeneously. 

In this paper, we present a Distributed Zero-Shot Learning (DistZSL) framework that can learn a GZSL model from multiple decentralized data sources.
Here, we assume each node does not have overlapping classes and we denote this distribution as a \textit{partial class-conditional distribution} (\textit{p.c.c.d.}). This assumption imposes more challenges for DistZSL: (1) different models would struggle to learn a consistent visual-to-attribute (V2A) mapping as each node learns V2A independently; (2) trained models on different nodes would also bias to local data; For instance, a model trained solely on birds with \textit{gray} and \textit{brown} wings may struggle to differentiate between \textit{blue} and \textit{black} wings in testing.  


To tackle the aforementioned challenges, we introduce two key components to our DistZSL. To be specific, we design a cross-node attribute regularizer to stabilize the distribution of attributes across nodes. Then, we present a global attribute-to-visual consensus to mitigate inconsistency among V2A mappings learned at different nodes. Note that both the cross-node attribute regularizer and the global attribute-to-visual consensus are applied on the client side. 

Our cross-node attribute regularizer is designed to enforce the distances between attribute features to be similar across different nodes. 
First, to estimate the inter-class attribute distances, the central node constructs a sparse similarity matrix using Graphical Lasso \cite{friedman2008sparse}.
Then, this similarity matrix is shared among individual nodes, acting as a cross-node reference during distributed training. 
Moreover, in local training, we employ the KL divergence to measure and minimize the distance between the predicted class-wise similarities and the constructed similarity matrix.

Our global attribute-to-visual consensus is introduced to mitigate the biased V2A mappings learned from different nodes. Since the attribute regularizer has stabilized the attribute feature distribution, a consistent V2A mapping can be achieved by learning an attribute-to-visual is achieved by establishing a bilateral connection between semantic and visual features. This strategy can improve the accuracy of attribute prediction and further mitigate local bias. In addition to predicting attributes from visual features, the bilateral connection reconstructs the visual features from the predicted attributes. We apply a bilateral loss on the differences between the reconstructed visual features and the original ones. By minimizing the bilateral loss, the reconstruction forces the predicted attributes to accurately maintain visual information. Thus, we can enhance the accuracy of the predicted attributes.

\begin{figure}
    \centering
    \includegraphics[width=0.7\linewidth]{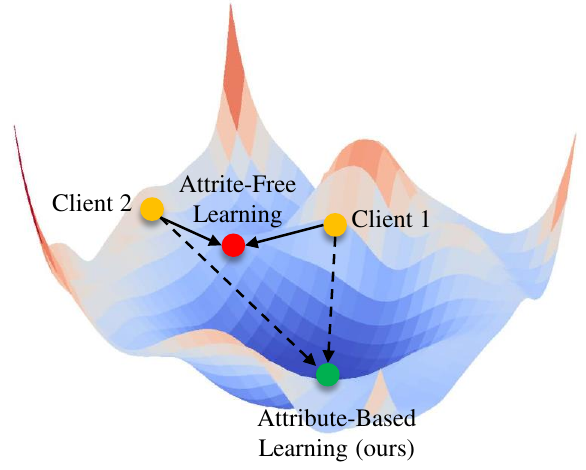}
    \caption{Attribute-based learning allows local models to learn towards the global minima across devices. In contrast, attribute-free learning simply averages the classifier weights of individual clients, leading to local optima. } 
    \label{fig:manifold}
    \vspace{-10pt}
\end{figure}

To evaluate our proposed DistZSL, we incorporate four state-of-the-art ZSL methods into six representative FL frameworks, resulting in 24 baseline models. 
Extensive experiments on three benchmark ZSL datasets demonstrate that our DistZSL consistently outperforms the baseline models.
In addition, we also evaluate DistZSL under different scenarios of distributed data, such as handling insufficient samples (\eg, 6 samples per class), fine-grained classes (\eg, birds), and an extensive number of classes, as well as different data distributions (\textit{i.i.d.}, \textit{non-i.i.d.} and \textit{p.c.c.d.}).
Comprehensive ablation studies also demonstrate the effectiveness of our proposed components in DistZSL.
To summarize,  the main contributions of this work are listed as follows:
\begin{itemize}
    \item We propose {DistZSL}, which infuses the ZSL ability into FL frameworks. We identify that attribute-based learning in ZSL can inherently benefit decentralized training.
    \item We pinpoint two critical challenges: decentralized data, and biased local updates. To address these issues, our proposed solution integrates a cross-device attribute regularizer and a bilateral semantic-visual connection. 
    \item  Through comprehensive experiments of our proposed method and various baselines on three ZSL datasets, we demonstrate the capability of addressing the identified challenges in various settings.
\end{itemize}

\section{Related Work}
\subsection{Zero-Shot Learning}
Zero-Shot Learning (ZSL) \cite{guo2020novel,ye2021disentangling,yang2021adaptive,zhang2019hierarchical,,chen2023zero,chen2021entropy,guo2024fine,chen2022federated} addresses a challenging problem in computer vision \cite{you2021domain,chen2021local,wang2023cal,chen2025cluster,yu2025dynamic,li2025pataug,zhao2025continual,wei2024plantseg,chen2019cycle,wei2024benchmarking,wang2025discrimination,you2022pixel,wei2024snap,zhang2024towards,chen2024cf,zhao2025synthetic,chen2025fastedit} where the test set contains additional classes not presented during training. 
To bridge the seen and unseen classes \cite{akata2015label,li2017zero,ye2021alleviating}, a standard solution is learning the visual-semantic relationships. Intermediate class-level semantic representations include attribute annotations \cite{lampert2013attribute}, natural language descriptions \cite{elhoseiny2013write}, etc.
In general, to learn the visual-semantic relationships, there are two streams of methods:  
embedding-based methods \cite{xu2020attribute,palatucci2009zero,norouzi2013zero} and generative methods \cite{xian2018feature,zhu2018generative,han2021contrastive}. 

The former group projects the visual and semantic information to the same feature space. The learned projection can then infer the class attributes for the samples of unseen classes.
Various methods have been proposed in this research direction, including graph learning \cite{liu2020attribute}, attentive learning \cite{liu2019attribute,zhu2019semantic}, similarity matching \cite{jiang2019transferable}, metric learning \cite{cacheux2019modeling,keshari2020generalized}, and meta-learning \cite{yu2020episode,verma2020meta}. 
The latter group first trains a generative model conditioning on semantic information. Then, the generative model can synthesize the visual features of the unseen classes. Finally, we can train a supervised classifier with the synthesized visual features.
Various generative models are applied to feature generation tasks, including generative adversarial nets \cite{xian2018feature,chen2020rethinking,li2019leveraging,chen2020canzsl}, variational autoencoders \cite{ma2020variational,chen2021semantics}, and invertible flows \cite{chen2022gsmflow,chen2021mitigating}. 

Generative approaches include two-stage training processes, \textit{i.e.,} generative model training and classifier training. They require retraining the classifier when involving more unseen classes. Moreover, training generative models are generally harder than discriminative models. Thus, to ease the overall training process, we follow the embedding-based paradigm.
Through the observations in Section \ref{exp:vsf}, we find that learning from more seen classes is beneficial to improve the visual-semantic generalization ability.
However, in the real world, most training classes are proprietary and not shared publicly due to privacy or confidentiality concerns. 
To utilize the locally seen classes, in this paper, we study Distributed Zero-Shot Learning ({DistZSL}) that learns from on-device data in non-identical class distribution as shown in Figure \ref{fig:fedzsl}.

\begin{figure*}
    \centering
    \includegraphics[width=0.93\linewidth]{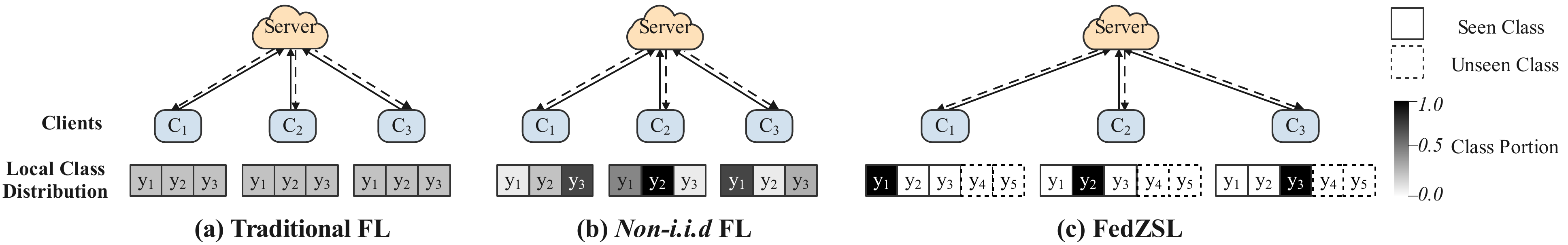}
    \vspace{-5pt}
    \caption{Data distributions of \textit{i.i.d.}, \textit{non-i.i.d.} and \textit{p.c.c.d.} settings. The darker color represents more training samples.}
    \label{fig:fedzsl}
    \vspace{-10pt}
\end{figure*}

\subsection{Federated Learning }  
Federated learning \cite{mcmahan2017communication} is a distributed learning protocol. It enables multiple participants to collaboratively learn a unified model without sharing the local data. 
Researchers in this area have been dedicated to improving efficiency and effectiveness, including the strategies for dealing with \textit{non-i.i.d.} data \cite{stich2018local,yu2018parallel}), preserving the privacy of user data \cite{bogdanov2012deploying}, ensuring fairness and addressing sources of bias \cite{li2019fair}, and addressing system challenges \cite{sheller2020federated}, multimodal data \cite{tan2023fedsea,ma2023fedsh,chen2023fedlive}.
Note that our proposed DistZSL is different from existing zero-shot related methods FL \cite{hao2021towards,zhang2021fedzkt,gudur_interspeech21}. These methods cannot directly generalize the federated model to unseen classes. Instead, the focus of these methods is to address the data and model heterogeneity problems. 

Specifically, Fed-ZD \cite{hao2021towards} considers improving model fairness on under-representative classes that only partial clients hold. They propose to perform data augmentation for those under representative classes. FedZKT \cite{zhang2021fedzkt} is proposed to solve the model heterogeneity by distilling the knowledge from heterogeneous local models. A global generative model is leveraged to distill the knowledge learned on local models. 
Fed-NCAC \cite{gudur_interspeech21} is inspired by the data impression technique \cite{nayak2019zero} that adapts the current model to emerging new classes. However, the adaption is based on training samples of new classes.
In this paper, we consider DistZSL with the \textit{p.c.c.d.} data, which intrinsically infuses the zero-shot learning ability into FL frameworks. 

\subsection{Data Heterogeneity} 
To address the data heterogeneity problem, various FL frameworks have been proposed, including FedProx \cite{li2020federated}, FedNova \cite{wang2020tackling}, Scaffold \cite{karimireddy2020scaffold}, MOON \cite{li2021model}. These methods typically involve specific aggregation policies to guide the global model learning and avoid shifting the learning direction. 
The concept of partial class-conditional distribution (\textit{p.c.c.d.}) was initially introduced in \cite{mcmahan2017communication} as a special type of \textit{non-i.i.d.} partition.
Further analysis and discussion distinguishing \textit{non-i.i.d.} from \textit{p.c.c.d.} were later presented in \cite{hsu2020federated}. 
To generate a \textit{non-i.i.d.} or \textit{p.c.c.d.} data distribution, we can draw a categorical distribution over the available training classes. Specifically, we can use a Dirichlet distribution $Dir(\alpha)$, with $\alpha$ being the concentration parameter controlling the non-uniformity of clients. 
Figure \ref{fig:fedzsl} (a) illustrates this: in an \textit{i.i.d.} setting, the concentration parameter is typically set to infinity. For a \textit{non-i.i.d.} setting, $\alpha$ is commonly set to $0.5$, while $\alpha$ is set to 0 in \textit{p.c.c.d.}. 
It is noted that the \textit{p.c.c.d.} setting is more difficult than the conventional \textit{non-i.i.d.} setting \cite{hsu2020federated}. Despite recognizing the challenges, no specific method was proposed to address these challenges.
In this paper, we directly address the difference between \textit{non-i.i.d.} and \textit{p.c.c.d.} settings. A DistZSL method is proposed to lessen the difficulties, achieving comparable performance in both settings.

\begin{figure*}
    \centering
    \includegraphics[width=0.95\linewidth]{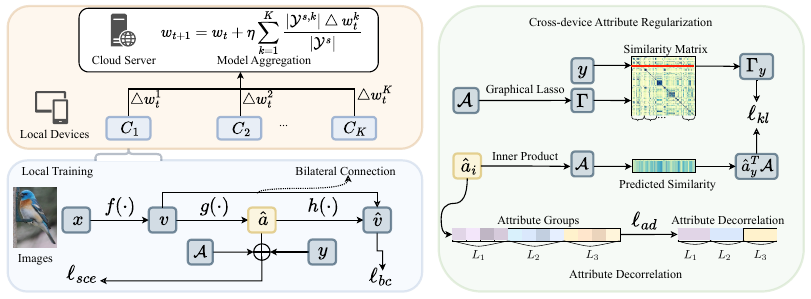} 
    \caption{An overview of the proposed  {DistZSL}, a decentralized framework for zero-shot learning models from multiple data sources with no exchange of local training data. On a local device, given an image sample $\vx_i$ from a class that is exclusive from all other devices, the image encoder $f(\cdot)$ produces the visual features $\vv_i$, which are further fed into the attribute regressor $g(\cdot)$ to predict the attributes $\widehat{a}_i$.
    In local training, we conduct attribute-based learning by (1) a visual-semantic alignment using semantic cross-entropy loss $\ell_{sce}$ to facilitate attribute prediction and an attribute decorrelation loss $\ell_{ad}$ to suppress the inter-class attribute occurrence, (2) a cross-device attribute regularizer $\ell_{kl}$ to stabilize attribute learning and avoid local models to be biased to locally available classes, and (3) a bilateral visual-semantic connection $\ell_{bc}$  to improve cross-device information consistency on the two modalities. 
     }
    \label{fig:fedzsl}
    \vspace{-10pt}
\end{figure*}

\section{Distributed Zero-Shot Learning}

\subsection{Overview of DistZSL}
We consider a distributed system of $K$ clients $C_1, C_2,\dots, C_K$. Each client owns a local data source for training, \textit{i.e.,} $\mathcal{D}^{s} = \{\mathcal{D}^{s,1}, \mathcal{D}^{s,2}, \ldots, \mathcal{D}^{s,K}\}$, where the superscript $s$ represents seen data that are available for training.
In particular, the $k$-th device has $N^k$ pairs of images with labels, \textit{i.e.,} $\mathcal{D}^{s,k}=\{(\vx_{i}^{s,k}, y_{i}^{s,k})\}_{i=1}^{N^{k}}$, where only a part of seen classes are observable $y_{i}^{s,k}\in \mathcal{Y}^{s,k}$. 
In addition to \textit{i.i.d.} and \textit{non-i.i.d.}, our study investigates a more practical yet challenging setting, namely partial class-conditional distribution (\textit{p.c.c.d.}), where multiple parties exclusively hold the training data from non-overlapping classes.
Notably, in \textit{p.c.c.d.} setting, the seen classes across the devices are non-overlapping, $\bigcap_{k \in [K]}\mathcal{Y}^{s,k}=\emptyset$ and $\bigcup_{k \in [K]} \mathcal{Y}^{s,k} = \mathcal{Y}^s$, where $|\mathcal{Y}^{s}|$ means the total number of the seen classes across $K$ devices. 
We follow the standard FL \cite{mcmahan2017communication} that each client trains a local recognition model based on the local data, while a central server collects the parameters periodically, and aggregates them to update the global parameters for recognizing both seen classes $\mathcal{Y}^s$ and unseen classes $\mathcal{Y}^u$. 
For brevity, we define $|\mathcal{Y}^{s}|+|\mathcal{Y}^{u}| = |\mathcal{Y}|$. 
To enable the parameter sharing between labels, the semantic information $\mathcal{A} = \{\va_{y} \}_{y}^{|\mathcal{Y}|}\in\mathbb{R}^{d_a\times |\mathcal{Y}|}$ is shared among all devices.
Formally, we leverage the training data on $K$ devices $\mathcal{D}^{s} \triangleq \bigcup_{k \in [K]} \mathcal{D}^{s,k}$ to initiate a unified model $w$. 
The global learning objective in \textit{p.c.c.d.} setting is to solve:
\begin{equation}
\begin{aligned}
  \underset{w}{\min} \mathcal{L}(w) = \sum_{k=1}^{K} \dfrac{|\mathcal{Y}^{s,k}|}{|\mathcal{Y}^s|} \mathcal{L}^{k}(w),
\end{aligned}
\end{equation}
where $\mathcal{L}^{k}(w) = \mathop{\mathbb{E}}_{(\vx,y)\sim \mathcal{D}^{s,k}} [\ell^{k}(w;(\vx,y))]$ is the empirical loss of the client $C_{k}$. 
We denote the model parameters at the round $t$ by $w_t$ and the $k$-th local model update by $\bigtriangleup w_{t}^k$. Therefore, the server will update the global model by aggregating $k$-th participant's local updates by:
\begin{equation}
\begin{aligned}
  w_{t+1} = w_t + \eta \sum_{k=1}^K\frac{|\mathcal{Y}^{s,k}|\bigtriangleup w_t^k }{|\mathcal{Y}^{s}|},
 \label{modelupdate}
\end{aligned}
\end{equation}
where $\eta$ is the learning rate on the server side. The overall training procedure can be found in Algorithm \ref{alg}.

\begin{algorithm}[t]
\hspace{-0.01em}\textbf{Input:} 
 clients number $K$, local datasets $\{\mathcal{D}^{s,1}, \ldots, \mathcal{D}^{s,K}\}$, scaling factor $\beta$, communication round $T$, local epochs $E$, local learning rate $\lambda$, global learning rate  $\eta$, loss coefficient $\mu$, batch size $B$    \\
\hspace{-0.5em}\textbf{Initialize:} 
Server model parameters $w^{0}$ \flushleft
\caption{Distributed Zero-Shot Learning ({DistZSL})}\label{euclid}
\begin{algorithmic}[1]
\State \textbf{Server executes:}
\For{$t = 0,1,...,T-1$}
    \State The server communicates $w_{t}$ to the \textit{i}-th client 
    \For{$k = 1,...,K$}
        \State The server communicates $w_{t}$ to the client $C_{k}$
        \State $\bigtriangleup w_t^k \leftarrow$ \textbf{PartyLocalTraining($k,w_t)$} 
    \EndFor
    \State $w^k_{t+1}$ $\leftarrow$  $w_t + \eta \frac{|\mathcal{Y}^{s,k}|\bigtriangleup w_t^k }{|\mathcal{Y}^{s}|}$
\EndFor

\State \textbf{PartyLocalTraining($k,w_t$):}
\State $w^{k}_t \leftarrow w_t$
\For{epoch $i = 0,1,\ldots,E$}
    \For{batch $\{(\vx_b,y_b)\}^{B}$}
        \State $ \{\vv_b\}^B \leftarrow \{f(\vx_b)\}^{B}$ \hfill {\color{blue} \# extract visual features}
        
        \State $\{\widehat{a}_b\}^{B} \leftarrow \{g(\vv_b)\}^{B}$ \hfill {\color{blue} \# generate attributes}
        \State $\{\widehat{v}_b\}^{B}\leftarrow\{h(\widehat{\va}_b)\}^{B}$ \hfill {\color{blue} \# reconstruct visual features}
        \State $\ell_{overall} = \ell_{sce} + \mu_1 \ell_{kl} + \mu_2 \ell_{bc} + \mu_3 \ell_{ad}$ 
        \State $w^{k}_t \leftarrow w^{k}_t - \lambda \bigtriangledown \ell^k$ 
    \EndFor
\EndFor
\State $\bigtriangleup w_t^k \leftarrow \beta(w^{k}_t - w_t)$
\State Return $\bigtriangleup w_t^k$ to server
\end{algorithmic}
\label{alg}
\end{algorithm}

\subsection{Local Training}\label{sec:local}
During the local training procedure, depicted in Figure \ref{fig:fedzsl}, client $C_k$ at the communication round $t$ receives the aggregated model weights $w_t$ from the central cloud and then applies it to the local model $w_t^k$. 
Given an input image $\vx$, a backbone network is leveraged as an image encoder $f(\cdot)$ to generate the visual features $\vv = f(\vx) \in \mathbb{R}^{d_v}$, where $d_v$ denotes the dimension of the visual features. 
Conventional classification models employ a fully connected layer on top of the backbone to produce class logits. In ZSL, we instead use an attribute regression layer $g(\cdot): \mathbb{R}^{d_v} \rightarrow  \mathbb{R}^{d_a}$ to derive the semantic attribute presence from visual features, given by $\widehat{\va} = g(f(\vx))\in\mathbb{R}^{d_a}$. Furthermore, the ground-truth semantic attributes serve as classifier anchors to generate the class logits. This unique learning strategy in ZSL leverages intermediate classifier anchors, and can significantly benefit FL. This is because it enables local models to learn toward the consistent feature manifold across devices. On the other hand, due to data heterogeneity in supervised FL, classifiers learned for different classes tend to be inconsistent across different clients.

\subsubsection{Attribute-Based Learning} 
In our setup, the attribute regression layer transforms the visual features, $\vv_i$, extracted from the image encoder into the attribute space. The pre-trained image encoder and the attribute regression layer are jointly optimized to enhance visual representation learning specifically for the ZSL task. Subsequently, the appropriate class for the predicted semantic attributes, $\widehat{\va}_i$, needs to be determined. Conventional FL trains a classifier for each class on each client. However, under the \textit{p.c.c.d.} setting, a client does not have access to training samples from other classes not owned by them, making it difficult to optimize the corresponding classifiers, thus leading to ill-posed classifiers. In contrast, our framework utilizes attribute vectors as classifier anchors to guide the local training process.
We perform a dot product operation between the predicted semantic attribute vector and the class-level ground-truth semantic attributes to compute the class logits. 
The semantic cross-entropy (SCE) loss is the objective to encourage the input images to have the highest compatibility score with their corresponding semantic attributes, which can be formulated as:
\begin{equation}
\begin{aligned}
  \ell_{sce} = - \sum_{\vx \in \mathcal{D}^{s,k}} \log \frac{\exp(\widehat{\va}_{y^{T}} \cdot \va_{y})}{\sum_{\va \in \mathcal{A}} \exp (\widehat{\va}_{y^{T}} \cdot \va)},
  \label{eq:sce}
\end{aligned}
\end{equation}
where \( \widehat{\va}_{y}\) represents the predicted attributes,  \(\va_{y}\) is the ground-truth attributes, and \(\mathcal{A} \) is the attributes of all classes. 

To improve the fidelity of the predicted attributes, we consider the nature of attributes in the semantic context. Specifically, certain semantic features are often collectively represented by multiple attributes, forming semantic groups. For instance, the two attributes \textit{grey wings} and \textit{blue wings} both describe the color of wings, and can be considered a semantic group collectively representing a single semantic concept. 
In practical scenarios, it is less likely for all attributes within the same semantic group to exhibit high responses. This observation motivates us to suppress the co-occurrence of the attributes within the same group. Following \cite{jayaraman2014decorrelating,xu2020attribute}, we mitigate the dependency between different semantic groups using an attribute decorrelation loss:
\begin{equation}
\begin{aligned}
 \ell_{ad} = \sum_{\vx \in \mathcal{D}^{s,k}} \sum_{l=1}^{L} \|  \widehat{a}^l \|_2,
\end{aligned}
\end{equation}
where we apply \(\ell_2\) norm on \( L \) groups of the semantic attributes. In essence, this decorrelation loss function serves to limit the magnitude of attribute values within each semantic group, effectively reducing the co-occurrence of attributes.

\subsubsection{Discussion on the Problem of Local Optima.}
In conventional FL, each client optimizes a classifier on its own local label space. Under heterogeneous or p.c.c.d. distributions, these local classifiers correspond to different sets of decision boundaries, so when the server aggregates model parameters, the global classifier is merely an average of inconsistent local optima. This averaging often leads to biased global solutions.

In contrast, attribute-based learning replaces per-client classifier weights with shared semantic anchors that are identical across all devices. Each client learns only a mapping from visual features to attributes, while the classifier itself is implicitly defined by the inner product with these shared anchors. As a result, local updates are guided toward the same semantic manifold regardless of class overlap. This eliminates the inconsistency in decision boundaries and ensures that the optimization landscape is aligned across clients.

Intuitively, attribute-based learning provides a common coordinate system in which local models can converge. Instead of averaging heterogeneous classifiers, aggregation combines visual-to-attribute mappings that are trained toward the same semantic targets. This alignment prevents clients from being trapped in incompatible local optima and enables the server to learn a coherent global model.
\vspace{-10pt}

\subsection{Cross-Device Attribute Regularization}\label{sec:att}
To deal with the data heterogeneity problem that causes biased data distribution across devices, the idea of our solution is to align the class relationships across devices, so that we can learn a consistent visual space.
However, as the training samples that involve visual information are strictly preserved in local devices, it is elusive to align the visual space with the collaborative classes. We propose a cross-device attribute regularizer to align the visual space in different local models according to the class similarities in the semantic space.
We start with constructing a class semantic similarity matrix. Graphical Lasso \cite{friedman2008sparse} is leveraged to estimate the sparse covariance of the semantic information $\mathcal{A}$ as the class semantic similarity matrix $\Gamma\in\mathbb{R}^{|\mathcal{Y}|\times|\mathcal{Y}|}$. 
Under the assumption that the inverse covariance $\Theta = \Gamma^{-1}$ is positive semidefinite, it minimizes an $\ell_{1}$-regularized negative log-likelihood:
\begin{equation}
\begin{aligned}
 \widehat{\Theta} = \underset{\Theta}{\argmin}~~\text{tr}(\mathcal{S}\Theta) - \log \text{det}(\Theta) + \delta\|\Theta\|_1,
\end{aligned}
\end{equation}
where $\mathcal{S}$ is a sample covariance matrix generated from $\mathcal{A}$, $\delta$ denotes the regularization parameter that controls the $\ell_1$ shrinkage. 
We further take the semantic similarity matrix as the probability distribution that the prediction logits of a training sample should match with. A natural way of learning the probability distribution is through knowledge distillation with logits matching\cite{hinton2015distilling}. 

The class similarity matrix $\Gamma$ is provided as the source knowledge to be transferred to target local models for learning visual features. We start with obtaining the soft targets by softening the peaky distribution of source and target logits with temperature scaling:
\begin{equation}
\begin{aligned}
&p_{\Gamma}(\cdot |y;\tau) = \softmax(\Gamma{y}/ \tau) = \frac{\exp(\Gamma{y}/ \tau)}{\sum^{|\mathcal{Y}|}_{y}\exp(\Gamma{y}/\tau )},
\\
&p_k(\cdot | \vx;\tau) = \softmax(\widehat{\va}_{y}^{T} \mathcal{A} / \tau) = \frac{\exp(\widehat{\va}_{y}^{T} \mathcal{A}/ \tau)}{\sum^{|\mathcal{Y}|}_{y}\exp(\widehat{\va}_{y}^{T} \mathcal{A}/\tau )},
\end{aligned}
\end{equation}
where $\tau$ is the temperature that can produce a softer probability distribution over classes with a high value. 
The knowledge distillation loss measured by the KL-divergence is:
\begin{equation}
\begin{aligned}
 \ell_{kl} &= \sum_{\vx \in \mathcal{D}^{s,k}} \text{KL}(p_{k}(\cdot | \vx ; \tau) \| p_{\Gamma}(\cdot | y; \tau)) \\
 &=\tau^2 \sum_{\vx \in \mathcal{D}^{s,k}}  p_{\Gamma}\log \frac{p_{\Gamma}}{p_k}.
\end{aligned}
\end{equation}


\subsection{Bilateral Visual-Semantic Connection} 
To more effectively model the relationships between visual and semantic modalities among local clients, we introduce a bilateral visual-semantic connection. This approach bolsters the mutual reinforcement between the two modalities from the global perspective. Previous studies \cite{li2021investigating,chen2020canzsl} has investigated related bilateral designs concerning the visual-semantic connection in generative ZSL. The aim is to couple the process of visual feature generation with a visual-to-semantic mapping. 

The challenge of modeling visual-to-semantic relationships arises due to data heterogeneity leading to biased local data. Beyond the existing visual-to-semantic learning, we further establish a semantic-to-visual regressor, thereby enhancing the learning model. The reconstructed visual features can provide feedback to assess the quality of the predicted attributes. In essence, if the predicted attributes truthfully reflect the original visual features, the reconstruction should be highly effective. 

From the global aggregation perspective, incorporating the semantic-to-visual knowledge gained from other parties can in turn benefit the local visual-to-semantic modeling. Specifically, given the visual features ${\vv}_i$ and the predicted attributes $\widehat{\va}_i$, we learn a semantic-to-visual transformation $h(\cdot): \mathbb{R}^{d_a} \rightarrow  \mathbb{R}^{d_v} $ that brings the predicted attributes from semantic space to the original visual space, yielding $\widehat{\vv}_i = h( \widehat{\va}_i)\in\mathbb{R}^{d_v}$. The predicted attributes are learned with the supervision from Eq. \ref{eq:sce}. Moreover, we add a bilateral connection loss to facilitate the learning of classifier anchors. The bilateral connection loss is applied between the extracted visual features and the reconstructed visual features:
\begin{equation}
\begin{aligned}
 \ell_{bc} = \sum_{\vx \in \mathcal{D}^{s,k}} \| h(\widehat{\va}_{y}) - f(\vx) \|_2,
\end{aligned}
\end{equation}
where $\widehat{\va}_{y}$ is the generated class-level attributes for image $\vx$.

\subsection{Joint Optimization} 

We define the overall local objective function in the client $C_k$ for DistZSL as follows:
\begin{equation}
\begin{aligned}
 \ell_{overall} = \ell_{sce} + \mu_1\ell_{bc} + \mu_2\ell_{kl} + \mu_3\ell_{ad},
\end{aligned}
\end{equation}
where $\mu_1,\mu_2$ and $\mu_3$ denote the coefficients of different loss functions. The local models are trained with the overall local objective for a few epochs. Following the training, the local update $\bigtriangleup w_t^k$ of client $C_k$ in the $t$-th communication round is submitted to the server for aggregation. The server combines these local updates to form the new global model, which is then distributed back to the clients for the next round training.

\subsection{Theoretical Analysis}

We provide theoretical support for the two key components in DistZSL, including the cross-node attribute regularizer $\ell_{\mathrm{kl}}$ (Eq.~(7)) and the global attribute-to-visual consensus $\ell_{\mathrm{bc}}$ (Eq.~(8)). The proof process is provided in supplementary materials. Throughout, classes have attribute prototypes $\mathcal{A}=\{\va_y\}_{y\in\mathcal{Y}}$, all probability vectors lie in the simplex $\Delta^{|\mathcal{Y}|-1}$, and softmax temperature $\tau>0$ is fixed.

\subsubsection{Setup and assumptions.}
Let $f:\mathcal{X}\!\to\!\mathbb{R}^{d_v}$ be the backbone, $g:\mathbb{R}^{d_v}\!\to\!\mathbb{R}^{d_a}$ the attribute regressor, and $h:\mathbb{R}^{d_a}\!\to\!\mathbb{R}^{d_v}$ the semantic-to-visual regressor. We define $w$ as the model parameters.
For a sample $(\vx,y)$ on client $k$, define logits $z_k(\vx)=\widehat \va_k(\vx)^{\top}A \in \mathbb{R}^{|\mathcal{Y}|}$ with $\widehat \va_k(\vx)=g(f(\vx))$, and the corresponding client distribution
\begin{equation}
p_k(\cdot\mid \vx;\tau)=\mathrm{softmax}\!\big(\vz_k(\vx)/\tau\big)\in\Delta^{|\mathcal{Y}|-1}.
\end{equation}
Let $\Gamma\in\mathbb{R}^{|\mathcal{Y}|\times|\mathcal{Y}|}$ denote the global semantic similarity matrix (estimated once on the server), and $p_\Gamma(\cdot\mid y;\tau)=\mathrm{softmax}(\Gamma_y/\tau)$ denote the target distribution for class $y$.

We make the following mild assumptions restricted to the data manifold $\mathcal{M}\subset\mathcal{X}$ in distributed learning setting.

\begin{itemize}
\item[A1] (Bi-Lipschitz decoder locally on $\mathrm{Im}(g\!\circ\! f)$). There exist constants $0<c_h\le L_h<\infty$ such that for all $\va_1,\va_2$ in a neighborhood of $\mathrm{Im}(g\!\circ\! f)$,
\begin{align}
c_h\|\va_1-\va_2\|\ &\le\ \|h(\va_1)-h(\va_2)\|\ \nonumber \\ &\le\ L_h\|\va_1-\va_2\|.
\end{align}
\item[A2] (Bounded reconstruction). Training with $\ell_{\mathrm{bc}}$ yields a uniform bound $\|h(g(f(\vx))) - f(\vx)\|\le \delta$ for all $\vx\in\mathcal{M}$ and some $\delta\ge 0$.
\item[A3] (Model smoothness near FedAvg iterate). For a fixed $\vx$, the mapping $w\mapsto z(\vx;w)$ (logits under parameters $w$) is $L_z$-Lipschitz in a neighborhood of the aggregated parameters $\bar w$, and softmax has Lipschitz constant $L_{\mathrm{sm}}(\tau)$ in logits, such that 
$\|z(\vx;w_1) - z(\vx;w_2)\| \le L_z\|w_1-w_2\|, \forall w_1, w_2 \in \mathcal{N}(w)$
and
$\|\softmax(\frac{z_1}{\tau}) - \softmax(\frac{z_2}{\tau})\| \le L_{\mathrm{sm}}(\tau)\|z_1-z_2\|$.
\item[A4] (Prototype separability). Prototypes are unit-normalized, $\|\va_y\|=1$, and have attribute margin $\Delta_y=\min_{y'\neq y}\|\va_y-\va_{y'}\|>0$.
\end{itemize}

\subsubsection{Cross-node attribute regularization}

Client $k$ minimizes the KL divergence to the global target
\begin{equation}
\ell_{\mathrm{kl}}^{(k)}(\vx,y)=\tau^2\,\mathrm{KL}\!\big(p_\Gamma(\cdot\mid y;\tau)\,\Vert\,p_k(\cdot\mid \vx;\tau)\big).
\end{equation}

\begin{lemma}[Client-level alignment]
\label{lem:client_alignment}
If $\mathbb{E}_{(\vx,y)}[\ell_{\mathrm{kl}}^{(k)}(\vx,y)]\le \varepsilon_k$ for some $\varepsilon_k > 0$ for client $k$, then for almost all $(\vx,y)$
\begin{equation}
\big\|p_k(\cdot\mid \vx;\tau)-p_\Gamma(\cdot\mid y;\tau)\big\|_1
\ \le\ \sqrt{\tfrac{2}{\tau^2}\,\varepsilon_k}.
\end{equation}
where $\varepsilon_k$ denotes the expected cross-node alignment error of client $k$, i.e.,
$\varepsilon_k = \mathbb{E}_{(\vx,y)}[\ell^{(k)}_{\mathrm{kl}}(\vx,y)]$.
Consequently, for any two clients $j,k$,
\begin{equation}
\big\|p_j(\cdot\mid \vx;\tau)-p_k(\cdot\mid \vx;\tau)\big\|_1
\ \le\ \sqrt{\tfrac{2}{\tau^2}\,\varepsilon_j}+\sqrt{\tfrac{2}{\tau^2}\,\varepsilon_k}.
\end{equation}
\end{lemma}





\begin{theorem}[Server-level guarantee under FedAvg]
\label{thm:server_alignment}
Let $\bar p(\cdot\mid \vx;\tau)=\sum_{k}\alpha_k\,p_k(\cdot\mid \vx;\tau)$ be the mixture of client distributions with FedAvg weights $\alpha_k=\frac{|{\cal D}_{s,k}|}{\sum_j |{\cal D}_{s,j}|}$. Then
\begin{align}
&\mathrm{KL}\!\big(\bar p(\cdot\mid \vx;\tau)\,\Vert\,p_\Gamma(\cdot\mid y;\tau)\big)\   \nonumber \\
& \le\ \sum_k \alpha_k\,\mathrm{KL}\!\big(p_k(\cdot\mid \vx;\tau)\,\Vert\,p_\Gamma(\cdot\mid y;\tau)\big).
\end{align}
If assumption A3 holds and the global model distribution $p(\cdot\mid x;\bar w,\tau)$ is within $\xi$, in $L1$ of $\bar p(\cdot\mid x;\tau)$, then
\begin{align}
& \mathrm{KL}\!\big(p(\cdot\mid \vx;\bar w,\tau)\,\Vert\,p_\Gamma(\cdot\mid y;\tau)\big)
\ \nonumber \\ &\le\ \sum_k \alpha_k\,\mathrm{KL}\!\big(p_k\,\Vert\,p_\Gamma\big)\ +\ C\,\xi,
\end{align}
for a constant $C$ depending only on $\tau$.
\end{theorem}

Theorem~\ref{thm:server_alignment} states that, as each client reduces its local $\ell_{\mathrm{kl}}$, the global model’s predictive distribution moves monotonically closer to the target semantic distribution $p_\Gamma$, up to the small averaging approximation. Hence, it aligns attribute similarity patterns across clients.

\subsubsection{Global Attribute-to-Visual Consensus}

The bilateral loss
\begin{equation}
\ell_{\mathrm{bc}}(\vx)=\|h(g(f(\vx)))-f(\vx)\|^2
\end{equation}
enforces that $h$ acts as an approximate left-inverse of $g\!\circ\! f$ on the data manifold.

\begin{lemma}[Information preservation via approximate left-inverse]
\label{lem:bilipschitz_lower}
Under A1–A2, for any $\vx_1,\vx_2\in\mathcal{M}$,
\begin{equation}
\|g(f(\vx_1))-g(f(\vx_2))\|
\ \ge\ \tfrac{1}{L_h}\,\|f(\vx_1)-f(\vx_2)\|\ -\ \tfrac{2\delta}{L_h}.
\end{equation}
\end{lemma}




\paragraph{Interpretation.}
Lemma 3 states that distances in the visual space cannot collapse under $g$ (up to a $2\delta$ slack) because the decoder $h$ approximately inverts $g$ on the image of $f$: enforcing $\ell_{\mathrm{bc}}(\vx)=\|h(g(f(\vx)))-f(\vx)\|^2$ small (small $\delta$) guarantees that attribute predictions $g(f(\vx))$ retain discriminative information from $f(\vx)$.


\begin{lemma}[Attribute error bound from reconstruction]
\label{lem:attr_error_bound}
Fix $(\vx,y)$ and assume A1–A2. Then
\begin{equation}
\|g(f(\vx)) - \va_y\|\ \le\ \tfrac{1}{c_h}\,\big(\, \|h(\va_y)-f(\vx)\| + \delta\,\big).
\end{equation}
In particular, if $h(\va_y)$ approximates the class center in visual space with error $\varepsilon_y=\|h(\va_y)-f(\vx)\|$, then $\|g(f(\vx))-\va_y\|\le (\varepsilon_y+\delta)/c_h$.
\end{lemma}




\begin{theorem}[Margin preservation for attribute-based classification]
\label{thm:margin}
Assume A1, A2, A4 and let $\varepsilon_y=\|h(\va_y)-f(\vx)\|$. If
\begin{equation}
\delta + \varepsilon_y\ <\ \tfrac{c_h}{2}\,\frac{\Delta_y^2}{\max_{y'\neq y}\|\va_y-\va_{y'}\|},
\end{equation}
then the attribute-based classifier using logits $s_{y'}=g(f(\vx))^{\top}\va_{y'}$ predicts the correct label $y$.
\end{theorem}

Theorem~\ref{thm:margin} shows that minimizing $\ell_{\mathrm{bc}}$ (small $\delta$) controls the deviation of predicted attributes from their class anchors, which in turn guarantees class-wise separation in the attribute-based classifier as long as prototypes are reasonably separated. Combined with Lemma~\ref{lem:bilipschitz_lower}, the bilateral connection prevents information loss from $f$ to $g(f(x))$ and stabilizes cross-device learning by keeping discriminative structure intact.

\section{Experiments}
\subsection{Experiment Setup}
\subsubsection{Datasets} 
Unlike most FL methods that evaluate datasets which are relatively elementary in the \textit{non-i.i.d} setting, such as CIFAR-10 \cite{krizhevsky2009learning}, MNIST \cite{lecun1998gradient}, SVHN \cite{netzer2011reading}, and FMNIST \cite{xiao2017fashion}, which typically contain a limited number of classes and abundant class samples, we in this paper extensively evaluate our method on five benchmark datasets designed specifically for zero-shot learning. These datasets present a much more significant challenge due to their complexity and diversity.
{Caltech-UCSD Birds-200-2011 (CUB)} \cite{wah2011caltech} is a fine-grained bird dataset containing 11,700 images representing 200 bird species, with each species annotated with 312 manually annotated attributes.
{Animals with Attributes 2 (AwA2)} \cite{xian2018zero} comprises 37,322 images from 50 different animal classes, each animal class in this dataset is described using 85 attributes.
{SUN Scene Recognition (SUN)} \cite{patterson2012sun} includes 14,340 images representing 717 different scenes, with each scene annotated with 102 attributes.

APY \cite{farhadi2009describing} consists of 15,339 images from 32 object categories. It is split into two parts: the aPascal subset, derived from the PASCAL VOC 2008 dataset, and the aYahoo subset, containing images collected from the Yahoo search engine. Each object category is annotated with 64 attributes that describe visual properties such as shape, color, and texture.
{DeepFashion} \cite{liu2016deepfashion} is a large-scale clothing dataset containing over 800,000 images spanning a wide range of clothing categories and styles. It includes 50 clothing categories, 1,000 descriptive attributes, bounding boxes, and landmark points for fashion items. For these datasets, we adopt the standard splits for seen and unseen classes as proposed in \cite{xian2018zero}, specifically, 150/50 for CUB, 40/10 for AwA2, and 645/72 for SUN. These datasets pose a particular challenge because they contain many classes with limited image samples per class. Furthermore, the classes in the CUB and SUN datasets are fine-grained, which is a significant challenge in the context of federated learning.

\subsubsection{Evaluation Metrics}
For evaluation purposes, we use the average per-class top-1 accuracy as the primary metric in both our conventional Zero-Shot Learning (ZSL) and Generalized Zero-Shot Learning (GZSL) experiments, as proposed by Xian \textit{et al.} \cite{xian2018zero}.
In the conventional ZSL setting, we only evaluate the accuracy of the \textit{unseen} classes, denoted as $Acc_{\mathcal{C}}$. These are classes that none of the participants have access to.

In the GZSL setting, we extend the evaluation to include both \textit{seen} and \textit{unseen} classes. We calculate the accuracy of the test samples from both these classes, represented as $Acc_{\mathcal{Y}^{s}}$ and $Acc_{\mathcal{Y}^{u}}$, respectively.
To gauge the performance of our method in the GZSL setting, we compute the harmonic mean $Acc_{\mathcal{H}}$ of the accuracies of the seen and unseen classes. This is calculated as follows:
\begin{equation}
Acc_{\mathcal{H}} = \frac{2*Acc_{\mathcal{Y}^s}*Acc_{\mathcal{Y}^u}}{Acc_{\mathcal{Y}^s}+Acc_{\mathcal{Y}^u}} .
\end{equation}
The harmonic mean provides a balance of the performance across the seen and unseen classes, helping us to avoid a bias towards the class type with a higher number of samples.

\begin{table*}[!h]
\caption{Performance comparisons (\%) on three datasets among FL baselines, ZSL baselines, and the proposed DistZSL in centralized and \textit{i.i.d.} settings. * represents ViT-based backbone.}
\centering
\vspace{-0.25em}
\setlength{\tabcolsep}{3pt}
\scalebox{0.84}{
\begin{tabular}[t]{c  l | cccc | cccc | cccc | cccc | cccc}
\toprule
&  &  \multicolumn{4}{c|}{CUB} & \multicolumn{4}{c|}{AwA2}  &  \multicolumn{4}{c}{SUN} &  \multicolumn{4}{c}{{{APY}}}&  \multicolumn{4}{c}{{{DeepFashion}}}  \\ 
&  &  Acc$_{\mathcal{C}}$ &  Acc$_{\mathcal{Y}^{u}}$ & Acc$_{\mathcal{Y}^{s}}$ & Acc$_{\mathcal{H}}$  & Acc$_{\mathcal{C}}$ & Acc$_{\mathcal{Y}^{u}}$ & Acc$_{\mathcal{Y}^{s}}$ & Acc$_{\mathcal{H}}$  & Acc$_{\mathcal{C}}$ & Acc$_{\mathcal{Y}^{u}}$ & Acc$_{\mathcal{Y}^{s}}$ & Acc$_{\mathcal{H}}$ & 
{{Acc$_{\mathcal{C}}$}} & {{Acc$_{\mathcal{Y}^{u}}$}} & {{Acc$_{\mathcal{Y}^{s}}$}} & {{Acc$_{\mathcal{H}}$}} & 
{{Acc$_{\mathcal{C}}$}} & {{Acc$_{\mathcal{Y}^{u}}$}} & {{Acc$_{\mathcal{Y}^{s}}$}} & {{Acc$_{\mathcal{H}}$}}
\\
\\[-6pt] \hline \\[-6pt]
\parbox[t]{2mm}{\multirow{6}{*}{\rotatebox[origin=c]{90}{\textit{\textbf{centralized}}}}} 
&    {APN  }  
    & 72.0  & 65.3  & 69.3  & 67.2
    & 68.4  & 56.5  & 78.0  & 65.5
    & 61.6  & 41.9  & 34.0  & 37.6
    &  38.7   &  18.9   &  43.7   &  26.3 
    &  35.2   &  25.6   &  34.3   &  29.3 
    \\
&    {GEM  }  
    & 77.8  & 64.8    & 77.1  & 70.4
    & 67.3  & 64.8    & 77.5  & 70.6
    & 62.8  & 38.1    & 35.7  & 36.9
    &  39.4   &  19.8     &  45.2   &  27.5 
    &  33.1   &  24.8     &  33.1   &  28.3 
    \\
&    {MSDN  }  
    & 76.1  & 68.7  & 67.5  & 68.1
    & 70.1  & 62.0  & 74.5  & 67.7
    & 65.8  & 52.2  & 34.2  & 41.3
    &  37.2   &  18.2   &  43.8   &  25.7 
    &  28.7   &  21.8   &  29.2   &  25.0 
    \\
&    {SVIP*  }
    &  79.8  &  72.1  &  78.1  &  75.0 
    &  69.8  &  65.4  &  87.7  &  76.9 
    &  71.6  &  53.7  &  48.0  &  50.7 
    &  41.3  &  23.9  &  37.1  &  29.1 
    &  36.2  &  29.8  &  30.1  &  30.0 
\\
\\[-8pt]
\cline{2-22}
\\[-6pt]
&  {DistZSL}  
    & 73.9          & 62.4              & 70.1             & 66.1
    & 68.5          & 61.0              & 71.2             & 65.7
    & 61.4          & 42.2              & 29.8             & 35.0
    &  38.3   &  18.5   &  44.9   &  26.2 
    &  34.6   &  24.2   &  34.7   &  28.5 
    \\
&  {DistZSL*}  
    &  82.4           &  71.3               &  76.8              &  73.9 
    &  66.7           &  64.3               &  84.7              &  73.1 
    &  72.4           &  53.8               &  47.4              &  50.4 
    &  40.7   &  21.8   &  45.4   &  29.5 
    &  35.9  &  26.2   &  33.8   &  29.5 
\\ 
\midrule
\parbox[t]{2mm}{\multirow{21}{*}{\rotatebox[origin=c]{90}{\textit{\textbf{i.i.d.}}}}} &
 {FedAvg}  
    & --            & --                & 41.1             & --
    & --            & --                & 90.1             & --
    & --            & --                & 0.5              & --
    & --            & --                &  63.2           & --
    & --            & --                &  38.4            & --
    \\
&  {+ APN  }          
    & 68.2          & 59.1              & 60.7            & 59.9
    & 54.5          & 38.9              & 76.2            & 51.5
    & 20.5          & 12.2              & 6.1             & 8.1
    &  34.2   &  8.9    &  47.5   &  15.0 
    &  26.1   &  16.3   &  \textbf{25.8} & 20.0 
    \\
&  {+ GEM  }          
    & 67.4          & 38.7              & 64.1            & 48.2
    & 61.3          & 28.6              & 78.5            & 42.0
    & 61.0          & 32.9              & 31.6            & 32.2
    &  35.1   & 11.1    &  45.8   &  17.9 
    &  25.5   &  19.7   &  23.2   &  21.3 
    \\
&  {+ MSDN  }          
    & 68.3          & 23.4              & 49.4            & 31.7
    & 57.0          & 17.9              & 70.6            & 28.5
    & 58.4          & 28.2              & 33.4            & 30.6
    &  30.1         &  9.8              &  49.8           &  16.4 
    &  20.3         &  6.5              &  24.1           &  10.3 
    \\
&  {+ SVIP*  }
    &  79.4  &  58.9  &  71.7  &  64.7  
    &  63.2  &  57.2  &  85.1  &  68.4 
    &  68.1  &  50.1  &  \textbf{46.4} &  48.2 
    &  38.2  &  12.6  &  48.2  &  20.0 
    &  25.1  &  26.5  &  32.3  &  29.1 
    \\
&  {MOON}         
    & --            & --                 & 41.0           & --
    & --            & --                 & 90.3           & --
    & --            & --                 & 0.6            & --
    & --            & --                 & 64.3            & --
    & --            & --                 & 39.5            & --
    
    \\
&  {+ APN  }          
    & 67.4          & 57.7              & 62.7            & 60.1
    & 55.3          & 37.5              & 85.4            & 52.1
    & 3.5           & 1.3               & 0.2             & 0.3
    & 33.8          &  9.5              &  48.6            &  15.9 
    & 24.6          &  16.6              &  24.1             &  19.7 
    \\
&  {+ GEM  }          
    & 66.4          & 34.3              & 66.2            & 45.2
    & 59.8          & 30.1              & 78.3            & 43.4
    & 59.6          & 25.2              & 35.0            & 29.3
    &  33.9           &  10.1               &  47.8             &  16.7 
    &  23.8           &  17.4               &  24.9             &  20.5 
    \\
&  {+ MSDN  }          
    & 68.4          & 24.7              & 49.6            & 33.0
    & 57.4          & 17.7              & 80.9            & 29.0
    & 59.2          & 28.8              & 31.7            & 30.2
    &  26.4           &  10.4               &  46.9             &  17.0 
    &  25.8           &  8.9               &  25.6              &  13.2 
    \\
&  {+ SVIP*  }
    &  79.6   &  57.1   &  69.3   &  62.6 
    &  64.1   &  57.7   &  85.1   &  68.7 
    &  69.1   &  51.4   &  45.2   &  48.1 
    &  36.7   &  10.0   &  48.3   &  16.6 
    &  25.3   &  26.1   &  32.2   &  28.8 
    \\
&    {FedGloss}        
    & --            & --               & 40.8            & --
    & --            & --               & 90.2            & --
    & --            & --               & 1.2             & --
    & --            & --               &  63.3             & --
    & --            & --               &  38.6             & --
    \\
&    {+ APN  }          
    & 67.0  & 58.4  & 59.8 & 59.1
    & 54.3  & 38.4  & 73.8 & 50.5
    & 32.4  & 17.6  & 21.4 & 19.3
    &  31.3   &  7.4    &  45.7  &  12.7 
    &  25.5   &  15.8   &  24.3  &  19.1 
    \\
&    {+ GEM  }          
    & 67.2  & 57.4  & 60.1  & 58.7
    & 55.4  & 39.4  & 74.3  & 51.5
    & 33.5  & 17.7  & 22.2  & 19.7
    &  31.4   &  8.1    &  46.5   &  13.8 
    &  26.0   &  14.7   &  22.8   &  17.9 
    \\
&    {+ MSDN  }          
    & 65.4  & 56.1  & 58.4  & 57.2
    & 51.8  & 34.2  & 71.9  & 46.4
    & 24.8  & 10.4  & 17.3  & 13.0
    & 24.1   &  7.5    &  37.4   &  12.5 
    &  19.4   &  9.4    & 30.1   &  14.3 
    \\
&    {+ SVIP* } 
    &  78.4  &  56.6  &  68.4  &  61.9 
    &  63.7  &  56.1  &  85.4  &  67.7 
    &  68.2  &  49.6  &  44.9  &  47.1 
    &  38.7  &  12.7  &  45.0  &  19.9 
    &  26.8  &  27.3  &  32.4  &  29.6 
    \\ 
\\[-8pt]
\cline{2-22}
\\[-6pt]
&   {DistZSL}
    & {71.0} & {61.6} & 62.1 & {61.8}
    & {59.7} & {52.7} & 74.5 & {61.8}
    & {63.3} & {43.3} & 29.6 & {35.2}
    &  {36.1} &  {11.8}  &  49.9  &  {19.1}  
    &  {27.2} &  {21.5}  &  24.1  &  {22.7}
    \\
&   {DistZSL*}
    &  \textbf{81.2} &  \textbf{57.5} &  \textbf{69.6} &  \textbf{63.0 }
    &  \textbf{65.8} &  \textbf{59.4} &  \textbf{85.5} &  \textbf{70.1}  
    &  \textbf{70.8} &  \textbf{53.4} &  45.8  &  \textbf{49.3}
    &  \textbf{40.3} &  \textbf{14.2} &  \textbf{56.6} &  \textbf{22.7}
    &  \textbf{28.4} &  \textbf{29.3} &  \textbf{32.4} &  \textbf{30.8}
    \\ 
\midrule
\parbox[t]{2mm}{\multirow{21}{*}{\rotatebox[origin=c]{90}{\textit{\textbf{Non-i.i.d.}}}}} 
&    {FedAvg}  
    & --            & --                & 6.4             & --
    & --            & --                & 18.4            & --
    & --            & --                & 1.9             & --
    & --            & --                &  21.1             & --
    & --            & --                &  12.8             & --
    \\
&    {+ APN  }          
    & 65.0          & 54.9              & 60.9            & 57.7
    & 53.7          & 41.9              & 76.2            & 54.1
    & 35.3          & 20.8              & 14.2            & 16.8
    &  28.1           &  11.4               &  36.9             &  17.4 
    &  25.1           &  12.8               &  20.1             &  15.6 
    \\
&    {+ GEM  }          
    & 67.2          & 37.9              & 62.8            & 47.3
    & 57.4          & 29.9              & 60.0            & 39.9
    & 60.2          & 30.8              & 33.1            & 31.9
    &  28.5           &  10.9               &  37.4             &  16.9 
    &  24.7           &  6.3                &  17.2             &  9.2 
    \\
&    {+ MSDN  }          
    & 64.8          & 25.3              & 40.5            & 31.2
    & 56.9          & 18.9              & 67.9            & 29.6
    & 57.6          & 29.4              & 32.3            & 30.8
    &  25.6           &  7.3                &  35.1             &  12.1 
    &  20.1           &  10.0               &  18.8             &  13.0 
    \\
&    {+ SVIP*  }
    &  75.2   &  52.4   &  68.9   &  59.5 
    &  59.7   &  51.2   &  70.8   &  59.4 
    &  66.1   &  48.6   &  43.8   &  46.1 
    &  34.0   &  12.1   &  41.3   &  18.7 
    &  25.5   &  15.0   &  26.6   &  19.2 
    \\
&    {MOON}          
    & --            & --                & 7.3           & --
    & --            & --                & 20.9          & --
    & --            & --                & 2.4           & --
    & --            & --                &  22.0           & --
    & --            & --                &  13.1           & --
    \\
&    {+ APN}          
    & 66.2          & 58.1              & 58.3          & 58.2
    & 54.9          & 41.6              & 78.5          & 54.4
    & 3.9           & 1.4               & 0.2           & 0.3
    &  30.0           &  12.3               &  35.1           &  18.2 
    &  24.9           &  11.4               &  24.7           &  15.6 
    \\
&    {+ GEM}          
    & 66.0          & 33.2              & 62.8          & 43.5
    & 58.1          & 28.9              & 62.5          & 39.5
    & 57.1          & 28.3              & 30.2          & 29.2
    &  30.3           &  12.7               &  34.7           &  18.6 
    &  25.3           &  12.5               &  25.0           &  16.7 
    \\
&    {+ MSDN}          
    & 67.6          & 27.9              & 36.5          & 31.6
    & 55.8          & 23.5              & 48.4          & 31.6
    & 58.2          & 29.8              & 31.6          & 30.7
    &  28.6           &  10.8               &  34.0           &  16.4 
    &  23.8           &  11.4               &  25.3           &  15.7 
    \\
    
&    {+ SVIP*}
    &  74.6   &  51.9   &  66.5   &  58.3 
    &  59.0   &  49.4   &  68.9   &  57.5 
    &  65.4   &  47.4   &  43.1   &  45.1 
    &  34.1   &  12.4   &  39.8   &  18.9 
    &  25.0   &  17.4   &  28.7   &  21.7 
    \\
&   FedGloss          
    & --            & --                & 7.7           & --
    & --            & --                & 19.7          & --
    & --            & --                & 2.6           & --
    & --            & --                &  23.7           & --
    & --            & --                &  12.9           & --
    \\
&  {+ APN  }          
    & 67.3          & 55.7              & 61.1          & 58.3
    & 53.4          & 34.6              & 81.0          & 48.5
    & 33.9          & 20.0              & 12.1          & 15.1
    &  28.4           &  12.3               &  37.4           &  18.5 
    &  25.2           &  12.9               &  21.4           &  16.1 
    \\
&  {+ GEM  }          
    & 68.0          & 34.8              & 65.4          & 45.4
    & 56.9          & 33.8              & 60.6          & 43.4
    & 60.3          & 28.1              & 35.0          & 31.2
    &  28.3           &  10.2               &  36.4           &  15.9 
    &  23.8           &  11.4               &  20.6           &  14.7 
    \\
&  {+ MSDN  }          
    & 67.6          & 21.3              & 55.6          & 30.8
    & 53.1          & 28.6              & 51.2          & 36.7
    & 55.8          & 21.3              & 32.4          & 25.7
    &  27.0           &  9.4              &  35.7        &  14.9 
    &  21.8           &  10.9             &  21.8         &  14.5 
    \\
&  {+ SVIP*}
    & 74.8  & 52.7  & 66.4  & 58.8
    & 59.9  & 50.6  & 69.7  & 58.6
    & 66.1  & 48.7   &  44.7   &  46.6 
    & 33.4  &  12.9   &  40.7   &  19.6 
    & 25.6  &  15.4  &  27.9  & 19.8
    \\ 
\\[-8pt]
\cline{2-22}
\\[-6pt]
&  {DistZSL}
    & {71.4} & \textbf{58.9}     & {62.0}   & {60.4}
    & {58.7} & {51.6}     & {70.0}   & {59.5}
    & {61.9} & {39.5}     & {30.7}   & {34.5}
    &  {34.8} &  {13.4}     &  {39.4}   &  {20.0}
    &  {26.5} &  {17.1}     &  {28.1}    &  {21.2}  
    \\
    & {DistZSL*}
    &  \textbf{80.3} &  54.3  &  \textbf{69.7} &  \textbf{61.0}
    &  \textbf{63.4} &  \textbf{53.8} &  \textbf{73.4} &  \textbf{62.1}
    &  \textbf{68.7} &  \textbf{52.4} &  \textbf{45.1} &  \textbf{48.5}
    &  \textbf{36.7} &  \textbf{14.5}  &  \textbf{50.7}  &  \textbf{22.6}
    &  \textbf{27.3} &  \textbf{17.6} &  \textbf{29.4} &  \textbf{22.0}
    \\ \midrule
\parbox[t]{2mm}{\multirow{21}{*}{\rotatebox[origin=c]{90}{\textit{\textbf{p.c.c.d.}}}}}&
  {FedAvg}   
    & --            & --                & 5.2           & --
    & --            & --                & 8.8           & --
    & --            & --                & 0.3           & --
    & --            & --                &  9.5            & --
    & --            & --                &  3.9            & --
    \\
&  {+ APN}          
    & 50.9          & 41.9              & 50.4          & 45.8
    & 33.1          & 24.9              & 29.9          & 27.2
    & 33.0          & 18.8              & 13.7          & 15.9
    &  17.1           &  10.7               &  24.5           &  14.9 
    &  15.8           &  10.6               &  8.0            &  9.1 
    \\
&  {+ GEM}          
    & 51.8          & 30.8              & 50.5          & 38.2
    & 43.0          & 19.5              & 33.9          & 24.7
    & 57.3          & 31.0              & 32.6          & 31.8
    &  17.2           &  11.9               &  26.1           &  16.3 
    &  16.8           &  12.6               &  10.4           &  11.4 
    \\
&  {+ MSDN}          
    & 49.7          & 19.9              & 20.1          & 20.0
    & 38.4          & 20.1              & 44.5          & 27.7
    & 53.8          & 25.7              & 27.3          & 26.5
    &  15.7           &  10.1               &  23.8           &  14.2 
    &  16.1           &  9.9                &  7.9            &  8.8  
    \\
&  {+ SVIP*}
    &  73.8           &  47.8               &  63.0           &  54.4 
    &  54.9           &  42.4               &  74.9           &  54.1 
    &  63.8           &  44.1               &  40.4           &  42.2 
    &  29.8           &  14.0               &  35.3           &  20.0 
    &  21.1           &  13.8               &  23.9           &  17.5 
    \\
&  {MOON}         
    & --            & --                & 6.1           & --
    & --            & --                & 8.9           & --
    & --            & --                & 0.2           & --
    & --            & --                &  8.1            & --
    & --            & --                &  3.7            & --
    \\
&  {+ APN  }          
    & 51.6          & 40.3              & 49.8          & 44.6
    & 34.8          & 27.1              & 32.3          & 29.5
    & 4.0           & 0.9               & 0.2           & 0.4
    &  17.2           &  10.1               &  14.3           &  11.8 
    &  15.8           &  10.6               &  9.3            &  9.9 
    \\
&  {+ GEM  }
    & 43.6          & 31.6              & 41.4          & 35.9
    & 44.9          & 28.3              & 32.9          & 30.4
    & 54.2          & 26.9              & 29.8          & 28.3
    &  17.6           &  10.4               &  13.4           &  11.7 
    &  16.4           &  10.2               &  9.1            &  9.6 
    \\
&  {+ MSDN  }          
    & 50.2          & 15.6              & 39.0          & 22.3
    & 33.2          & 25.4              & 60.6          & 33.8
    & 54.4          & 26.9              & 25.9          & 26.4
    &  17.0           &  9.4                &  12.1           &  10.6 
    &  17.3           &  13.1               &  11.4           &  12.2 
    \\
& {+ SVIP*}
    &  71.6   &  49.7   &  63.8   &  55.9 
    &  54.2   &  43.1   &  73.9   &  54.4 
    &  64.9   &  46.9   &  43.8   &  45.3 
    &  28.7   &  12.8   &  35.1   &  18.8 
    &  20.7   &  13.2   &  24.0   &  17.0 
    \\
&    {FedGloss}        
    & --            & --                & 6.0           & --
    & --            & --                & 9.0           & --
    & --            & --                & 0.3           & --
    & --            & --                &  10.3           & --
    & --            & --                &  4.8            & --
    \\
&    {+ APN  }          
    & 51.0          & 40.1              & 48.6          & 43.9
    & 33.8          & 30.2              & 37.4          & 33.4
    & 33.8          & 18.3              & 13.7          & 15.7
    & 17.0          & 9.5               & 25.3          & 13.8 
    & 15.9          & 11.2              & 7.0           & 8.6 
    \\
&    {+ GEM  }          
    & 55.4          & 30.1              & 54.8          & 38.9
    & 44.2          & 26.1              & 31.9          & 28.7
    & 54.9          & 27.4              & 28.4          & 27.9
    & 17.1          & 11.4              & 26.8          & 16.0 
    & 17.1          & 10.4              & 10.3          & 10.3 
    \\
&    {+ MSDN  }          
    & 53.4          & 16.4              & 41.4          & 23.5
    & 43.8          & 11.4              & 53.2          & 18.8
    & 55.1          & 25.8              & 29.1          & 27.4
    & 16.4          & 8.3               & 18.9          & 11.5 
    & 15.1          & 9.3               & 8.3           & 8.8 
    \\
&    {+ SVIP*}
    &  71.3   &  48.4   &  64.2   &  55.2 
    &  54.8   &  42.1   &  75.1   &  54.0 
    &  65.4   &  47.2   &  43.7   &  45.4 
    &  29.4   &  13.1   &  34.8   &  19.0 
    &  21.3   &  14.8   &  24.3   &  18.4 
\\ 
\\[-8pt]
\cline{2-22}
\\[-6pt]
&  {DistZSL}  
    & {71.6} & \textbf{57.5}     & {58.0}   & {57.8}
    & {57.2} & {45.5}     & {62.3}   & {52.6}
    & {60.9} & {39.9}     & 27.2     & {32.3}
    &  {19.2} &  \textbf{16.1}      &  {29.6}    &  {20.8}
    &  {23.2} &  {16.2}            &  {15.6}    &  {15.9}
    \\
    & DistZSL*
    &  \textbf{79.5}    &  50.7         &  \textbf{69.6}      &  \textbf{58.7}
    &  \textbf{57.3}    &  \textbf{46.0}   &  \textbf{79.7}   &  \textbf{58.3}
    &  \textbf{67.8}    &  \textbf{51.7}  &  \textbf{44.2}   &  \textbf{47.7}
    &  \textbf{31.3}        &  15.4         &  \textbf{45.9}      &  \textbf{23.0}
    &  \textbf{23.4}     &  \textbf{16.6}    &  \textbf{26.8}     &  \textbf{20.5}
    \\
\bottomrule
\end{tabular}}
\label{performance1}
\vspace{-1em}
\end{table*}

\subsubsection{Implementation Details} 
DistZSL and the baseline models were implemented using PyTorch. Our code base is built on MOON \cite{li2021model}, a platform that has integrated the Federated Learning (FL) baselines for supervised learning.
For the implementation of incorporating Zero-Shot Learning (ZSL) baselines into FL frameworks, we referred to the official implementations of various models: APN \cite{xu2020attribute}, GEM \cite{liu2021goal}, and MSDN \cite{chen2022msdn}.
For FL training, we used an SGD optimizer with a weight decay of 1e-5 and a momentum of $0.9$. The number of communication rounds and the default number of participants were set to 100 and 10, respectively, unless specified otherwise. The batch size and number of local epochs in each communication round were set to 64 and 2.
In all our experiments, we utilized a pre-trained CNN network, ResNet101, as the backbone. For hardware, all the experiments were conducted on a Lenovo workstation equipped with two NVIDIA A6000 GPUs.

\subsubsection{Baselines} 
We choose three representative state-of-the-art embedding-based ZSL methods as our baselines, including  APN  \cite{xu2020attribute},   GEM  \cite{liu2021goal},  MSDN  \cite{chen2022msdn} and SVIP \cite{chen2025svip}. In addition, we fit the three ZSL baselines into five representative federated learning frameworks, including   FedAvg  \cite{mcmahan2017communication},   FedProx  \cite{li2020federated},   FedNova  \cite{wang2020tackling},   Scaffold  \cite{karimireddy2020scaffold},  MOON  \cite{li2021model}, and FedGloss \cite{caldarola2025beyond}. 
\vspace{-10pt}
 
\subsection{Main Results}\label{sec:exp}
We conducted experiments using data sampled from both \textit{i.i.d.} and \textit{non-i.i.d.} distributions to compare with the model trained in our \textit{p.c.c.d.} setting. In addition, we report the performance results in the \textit{centralized} setting, where the training is conducted on a single device. Although our approach is not specifically designed for the centralized setting, we can still achieve competitive results among state-of-the-art methods. For the \textit{i.i.d.} distribution, we evenly distributed the set of seen classes across ten devices, while for the \textit{non-i.i.d.} distribution, we sampled the data partition using a Dirichlet distribution $Dir(\alpha)$ with a concentration parameter $\alpha=0.5$. Table \ref{performance1} presents the performance comparisons between the federated baselines and our proposed methods for zero-shot recognition in the \textit{i.i.d.}, \textit{non-i.i.d.}, and \textit{p.c.c.d.} settings, respectively. The complete results of all six FL baselines are presented in Table~\ref{performance2} and \ref{performance3} in the Appendix. The results show that direct prediction of class labels using FL is only effective for the CUB and AwA2 datasets under the \textit{i.i.d.} setting. For the SUN dataset, the global model training could not converge. The ZSL baselines APN, GEM, MSDN and SVIP that allow attribute-based learning significantly improve performance across all three settings and all datasets. Furthermore, our DistZSL with the proposed cross-device attribute regularizer and bilateral semantic-visual connection can further enhance the performance of the baseline methods to a significant extent.

\vspace{-5pt}
\subsection{Ablation Study}
In this ablation study, we systematically evaluate various stripped-down versions of our full proposed model to validate individual components of the proposed DistZSL. 
In Table \ref{ablation}, we report the ZSL and GZSL performance results of each version on CUB dataset. 
In the FedAvg, we present the results obtained from attribute-free learning using FedAvg on the 150 seen classes. The intricate nature of fine-grained recognition results in failed global aggregation, consequently leading to poor performance. The introduction of classifier anchors, however, ensures the global model is uniformly optimized towards a common direction, thus the performance is significantly improved. The addition of individual components yields consistent performance gains, as seen in the results. The best performance is achieved when the KL loss function $\ell_{kl}$, the bilateral semantic-visual connection $\ell_{bc}$, and attribute decorrelation $\ell_{ad}$ are all applied. 


\begin{figure}[t]
    \centering
    \subfloat[]
    {\includegraphics[width=0.6in]{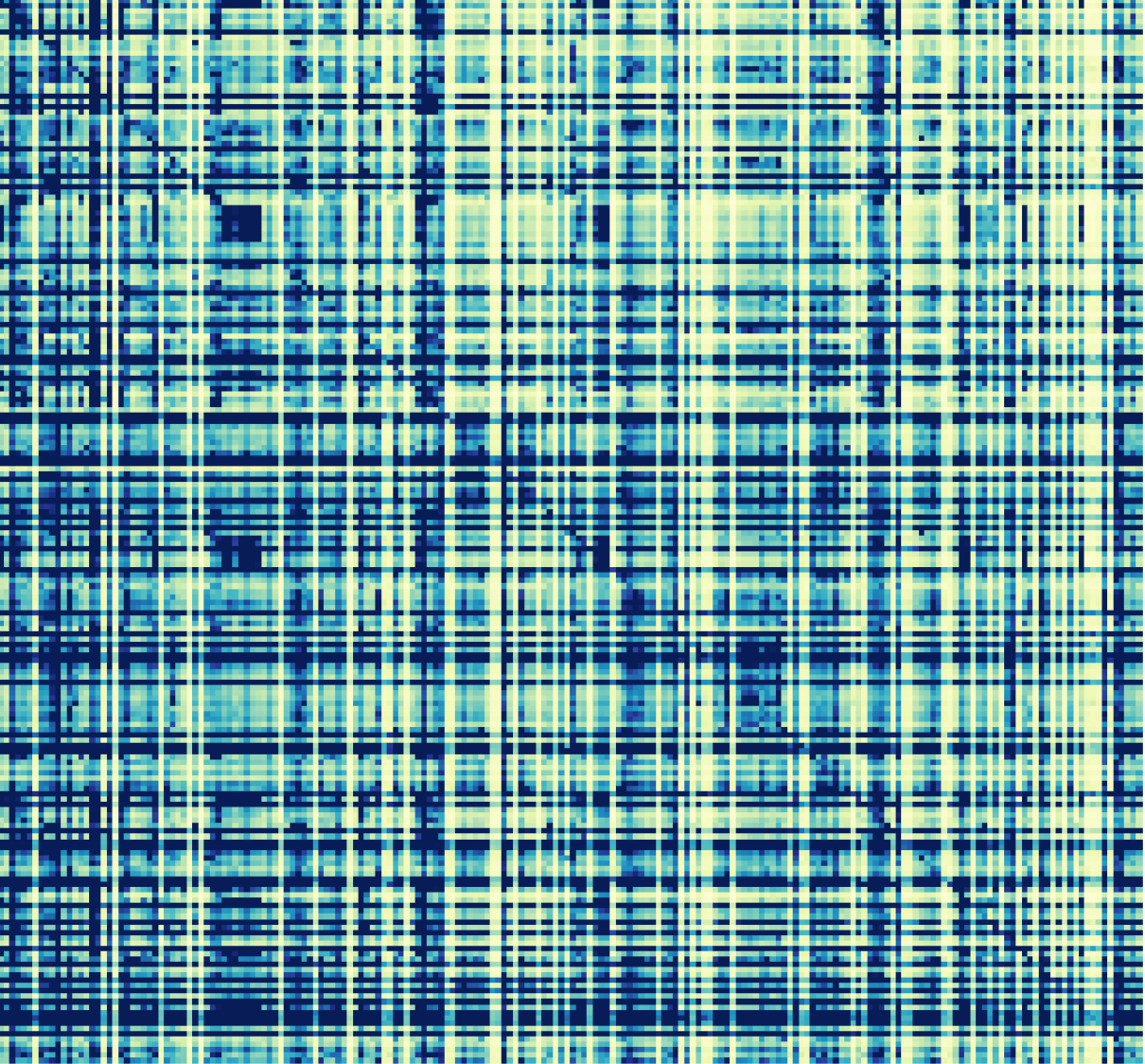}}\hspace{0.5pt}
    \subfloat[]
    {\includegraphics[width=0.6in]{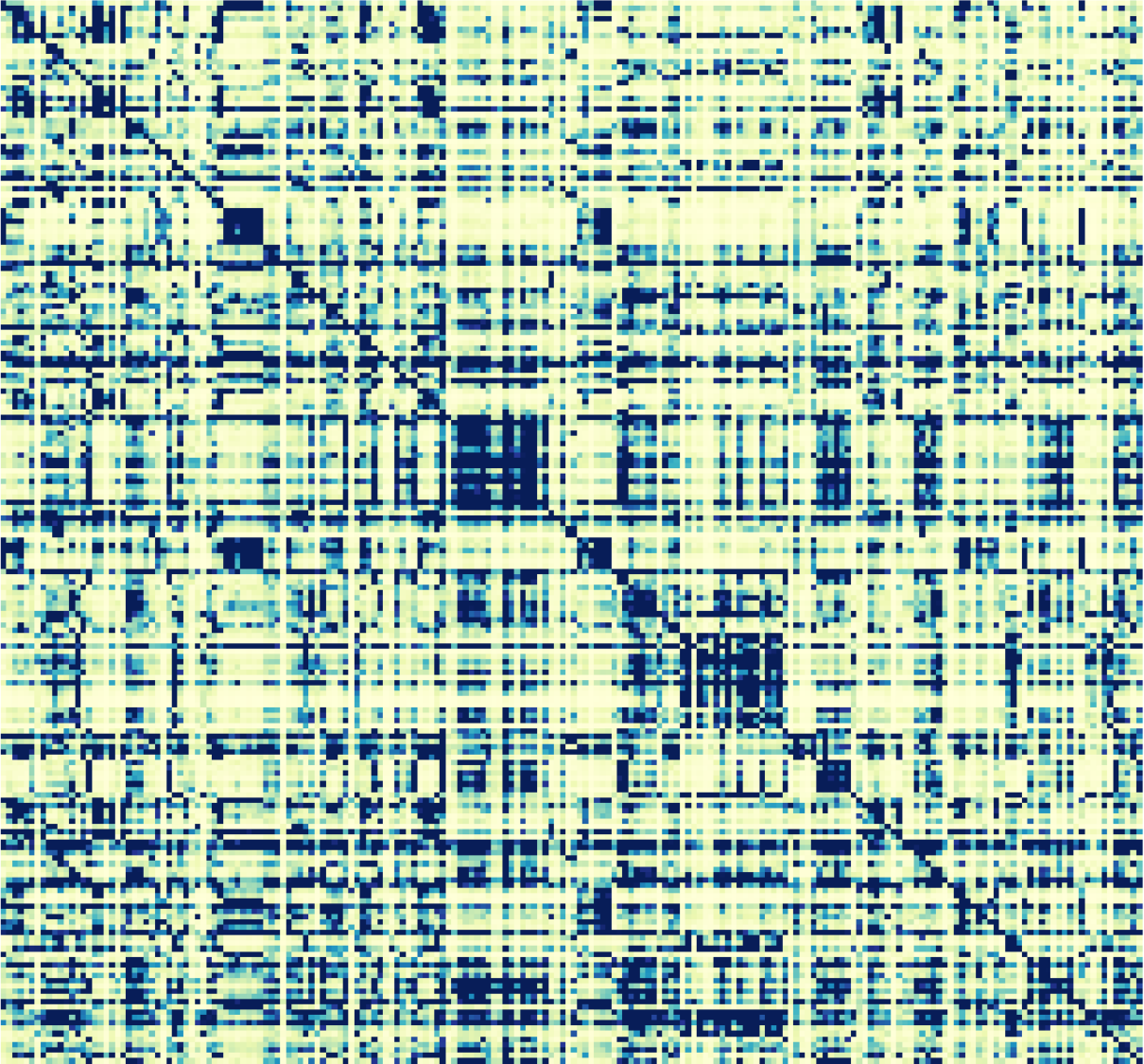}}\hspace{0.5pt}
    \subfloat[]
    {\includegraphics[width=0.6in]{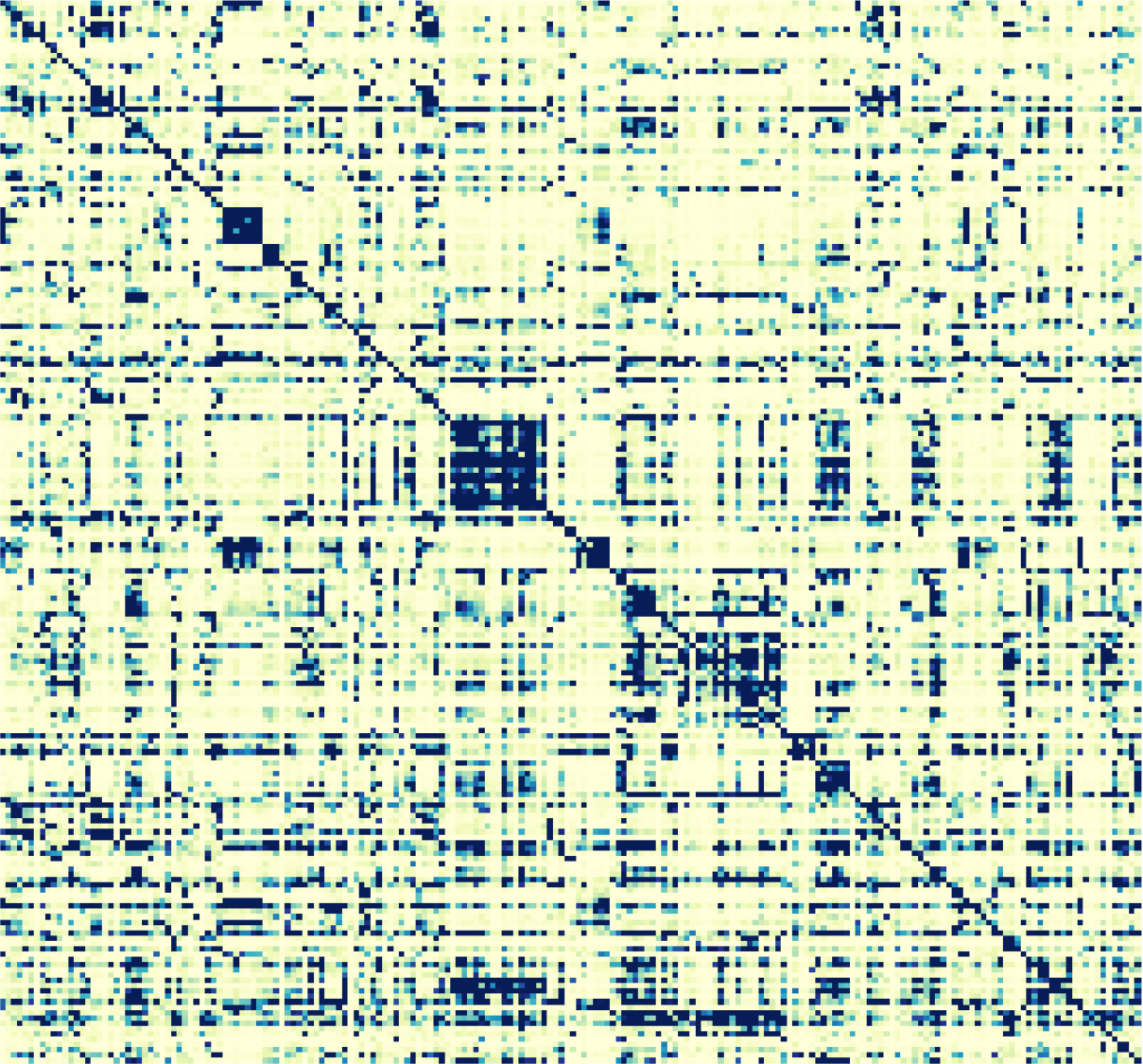}}\hspace{0.5pt}
    \subfloat[]
    {\includegraphics[width=0.6in]{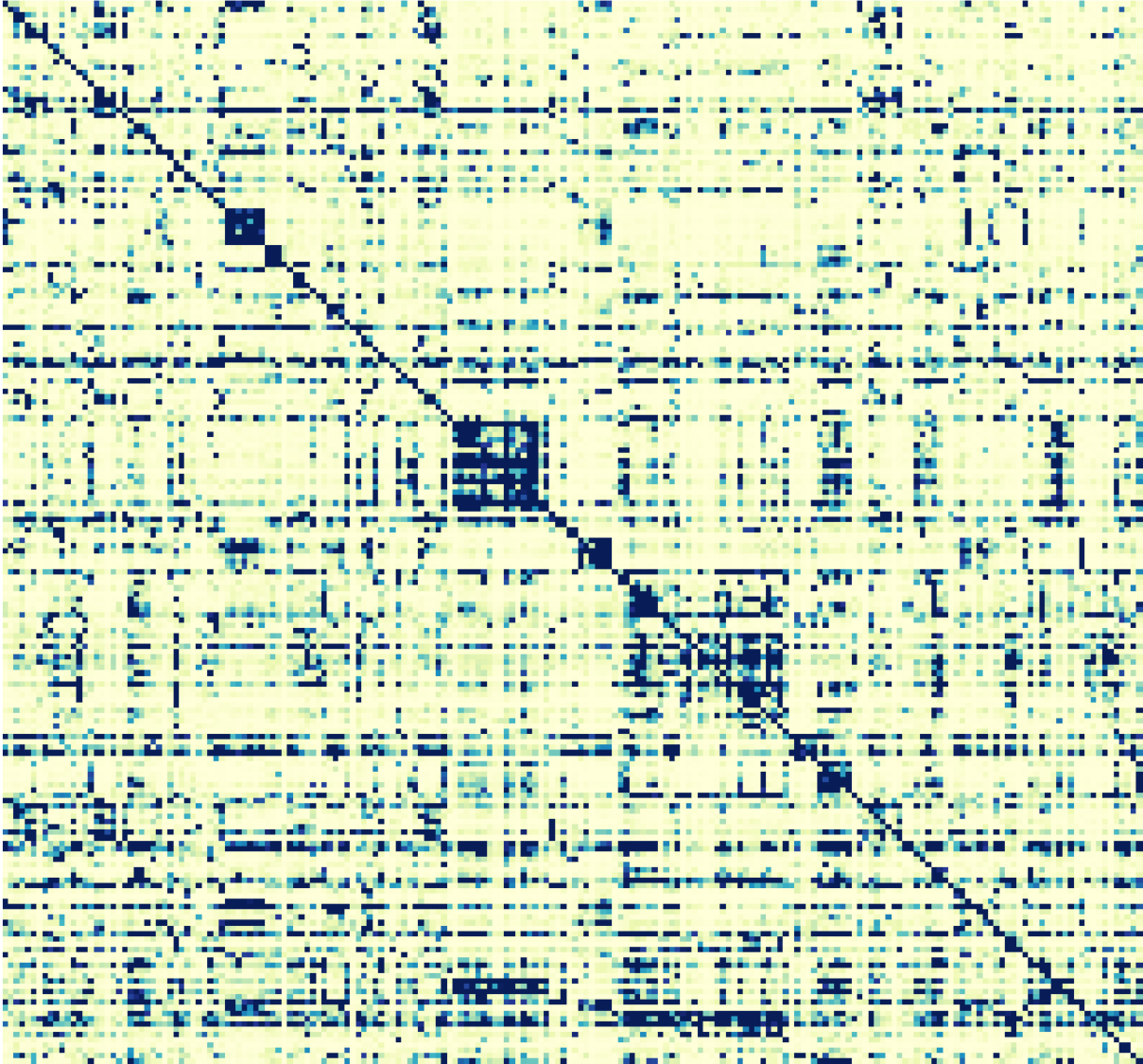}}\hspace{0.5pt}
    \subfloat[]
    {\includegraphics[width=0.65in]{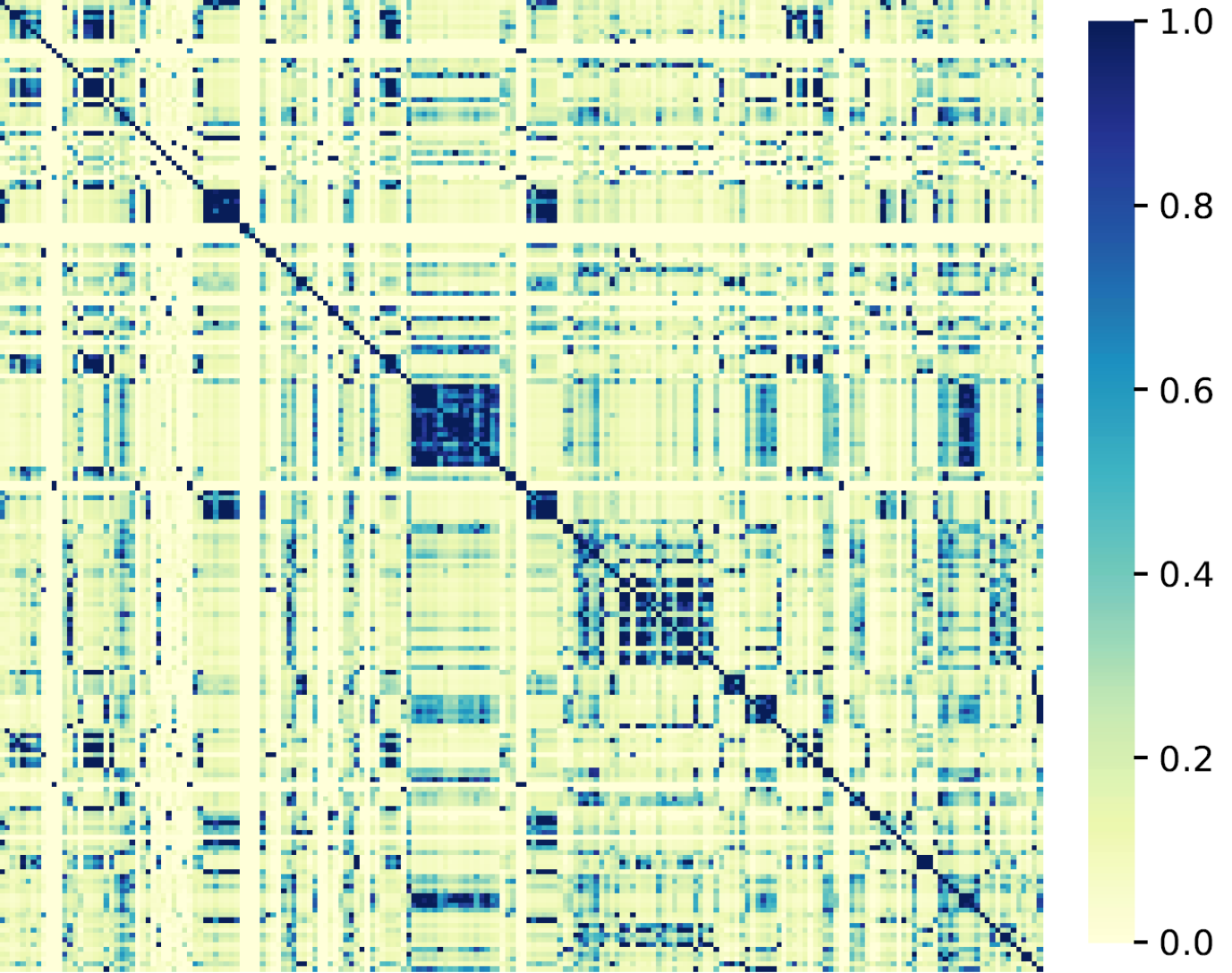}}\\
    \caption{Averaged similarities between the predicted attributes and the ground-truth attributes on CUB test samples. (a)-(d) illustrate similarities after the first, fifth, tenth, and twentieth communication round; (e) shows the pre-computed similarity matrix described in Section \ref{sec:att}.}\label{perturbation_vis}
    \vspace{-5pt}
    \label{heatmap_vis}
\end{figure}

\begin {table}[t]
\caption {Effects of different components on CUB dataset with various stripped-down versions of DistZSL.}
\begin{center}
\scalebox{1.01}{
\begin{tabular}[t]{lc ccc }
\toprule
    & Acc$_{\mathcal{C}}$ &  Acc$_{\mathcal{Y}^{u}}$ & Acc$_{\mathcal{Y}^{s}}$ & Acc$_{\mathcal{H}}$  \\
 \hline 
  {FedAvg}  
    & -             & -                 & 5.2             & -
    \\
 \textit{+ Attribute-Based Learning}  
    & 58.9          & 48.7              & 53.2            & 50.9
    \\

  { {DistZSL} w/ $\ell_{bc}$}      
& {63.7} &{51.3} & {56.1} & {53.7}
\\
  { {DistZSL} w/ $\ell_{kl}$}      
& {65.8} &{50.4} & {59.5} & {54.6}
\\
  { {DistZSL} w/ $\ell_{ad}$}      
& {62.1} &{49.2} & {56.6} & {52.6}
\\

  { {DistZSL} w/ $\ell_{bc} + \ell_{kl}$}      
& {71.2} &{56.3} & {55.7} & {56.0}
\\
  { {DistZSL} w/ $\ell_{ad} + \ell_{kl}$}      
& {68.6} &{54.2} & {57.1} & {55.6}
\\
  { {DistZSL} w/ $\ell_{bc} + \ell_{ad}$}      
& {70.1} &{55.0} & {58.1} & {56.5}
\\
  { {DistZSL}}      
& \textbf{71.6} & \textbf{57.5}     & \textbf{58.0}   & \textbf{57.8}
\\
\bottomrule
\end{tabular}}
\end{center}
\label{ablation}
\vspace{-15pt}
\end {table} 

\subsubsection{Impact of Attribute-Based Learning} 
We investigate the difficulties encountered in executing fine-grained recognition within FL frameworks. The datasets we chose for these experiments, namely CUB, SUN, and AwA2, are renowned for their fine-grained attributes. With 50 unique animal categories, even the AwA2 dataset is regarded as fine-grained when compared with more general datasets.
The experiments are conducted on five federated learning frameworks, all trained exclusively on data from seen classes in a supervised setting. As presented in Table \ref{performance1}, our results indicate that the (\textit{p.c.c.d.}) data prevents all FL frameworks from effectively learning a robust global model suited for fine-grained class recognition. However, our proposed method, DistZSL, attains an accuracy of 58.0\% on seen classes for the CUB dataset, significantly outperforming conventional FL frameworks which reach a maximum accuracy of just 6.2\%. For the SUN dataset, which contains a vast 717 classes, the results are even more discouraging, with the highest recorded accuracy of only 0.3\% on seen classes.

Under \textit{non-i.i.d.} conditions, as shown in Table \ref{performance1}, the outcomes mirror the previous findings. A direct comparison between our method (62.0\%) and Scaffold (7.7\%) exemplifies this. In Table \ref{performance1}, even under the assumption of \textit{i.i.d.} data, FL has limited success with the CUB and SUN datasets.
These outcomes confirm that fine-grained recognition is a significant challenge in FL. However, by incorporating attribute-based learning, we can address this issue. As illustrated in Figure \ref{fig:manifold}, classifier anchors can function as local references for other clients, helping them to learn consistent visual features across different classes towards global optima.


\begin{figure}[t]
	\centering
	\subfloat[FedAvg]
	{\includegraphics[width=1.1in]{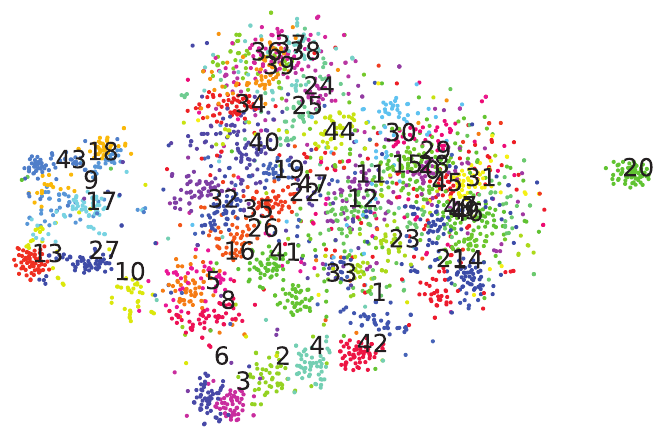}}
	\subfloat[FedAvg + Attribute-Based Learning]
	{\includegraphics[width=1.1in]{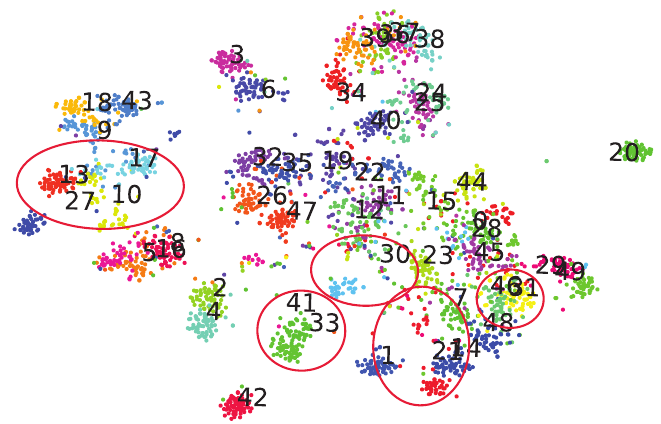}}
	\subfloat[DistZSL]
	{\includegraphics[width=1.1in]{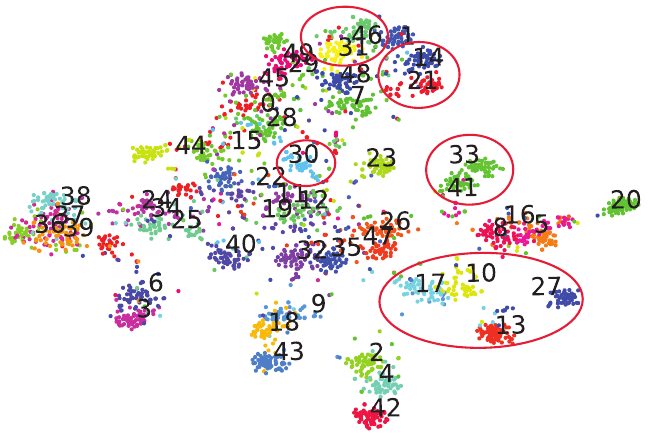}}\\
	\caption{t-SNE visualization with vanilla FedAvg, FedAvg with attribute-based learning, and FedAvg with proposed cross-device regularizer and bilateral visual-semantic connection. The numbers are annotated on the mean points of class distributions. The red circles highlight some improvements of DistZSL on attribute learning.}\label{prec_cifar}
    \vspace{-10pt}
\end{figure}

\subsubsection{Impact of Cross-Device Attribute Regularizer} 
To investigate the impact of the cross-device attribute regularizer, we visualize the averaged similarities between the predicted attributes and the ground-truth attributes on the CUB dataset. It is worth noting that the similarities are strictly generated by the test samples. As shown in Figure \ref{heatmap_vis}, (a)-(d) are the similarity maps generated after training for one, five, ten, and twenty communication rounds. (e) represents the predefined similarity matrix $\Gamma$. It is clear that class-wise semantic similarity in DistZSL can finally converge to the same patterns as the predefined similarity matrix shows. The visualization of prediction similarities confirms the effectiveness of the proposed cross-device attribute regularizer.

\subsubsection{Feature Visualization} 
To further validate the effectiveness of attribute-based learning with classifier anchors and the proposed DistZSL, we visualize the visual features produced by (a) FedAvg, (b) FedAvg with attribute-based learning, and (c) the complete DistZSL framework. The visual features are extracted from the 50 unseen classes on the CUB dataset. The numbers denote the Euclidean center points of visual features of each class.
The vanilla FedAvg cannot learn discriminative visual representations of the unseen classes. In contrast, attribute-based learning allows the model to extract semantic information on specific attributes, making the visual features on unseen classes considerably distinguishable. In DistZSL, the learned visual representations are further improved. We highlight some improved cases in the red circles. For example, the features of class 30 in attribute-based learning are dispersed, whereas in DistZSL they are more concentrated.


\begin{figure}[t]
    \centering
    \vspace{-10pt}
    \subfloat[\textit{i.i.d.} partition]
    {\includegraphics[width=1.13in]{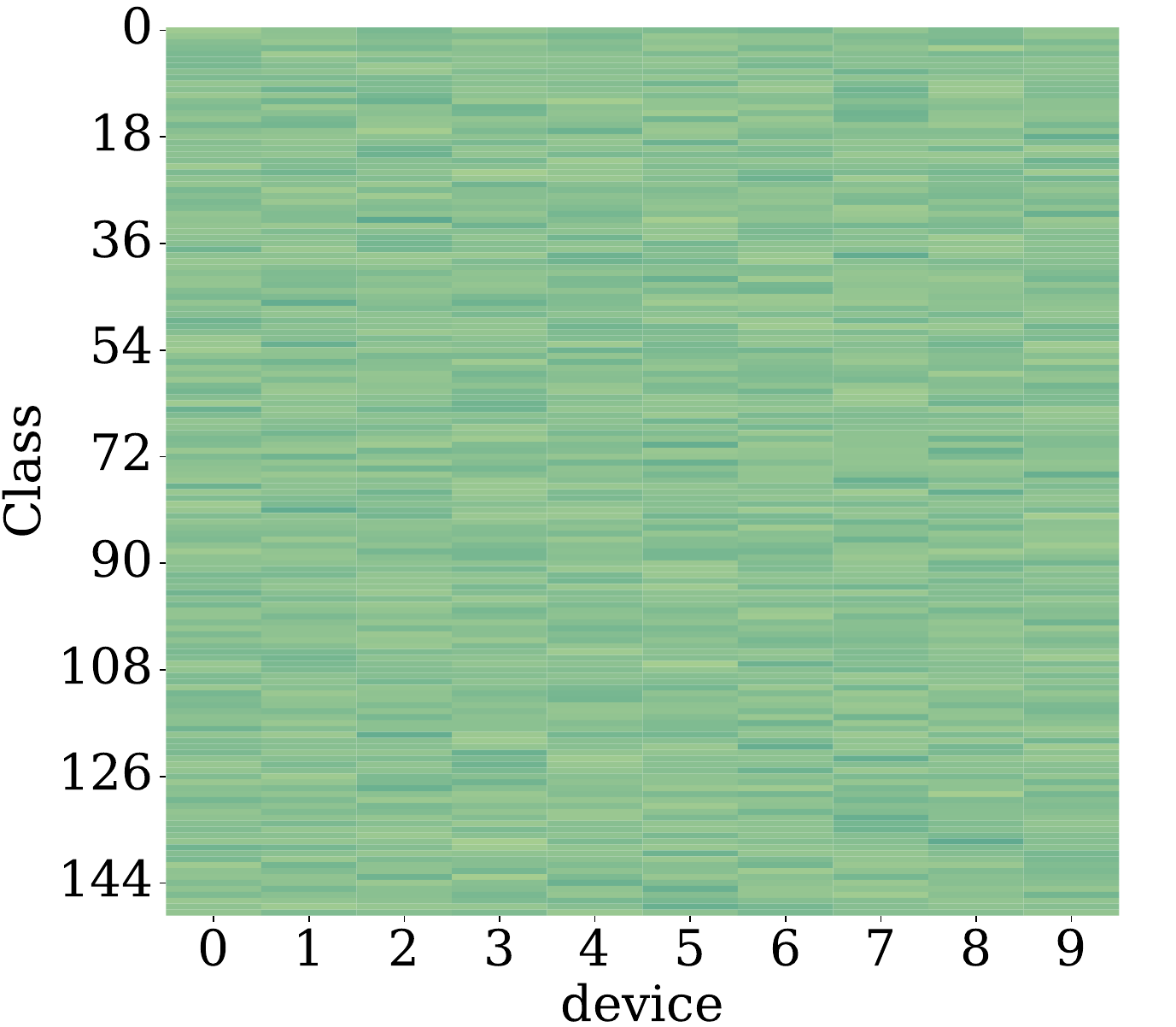}}
    \subfloat[\textit{non-i.i.d.} partition]
    {\includegraphics[width=1.13in]{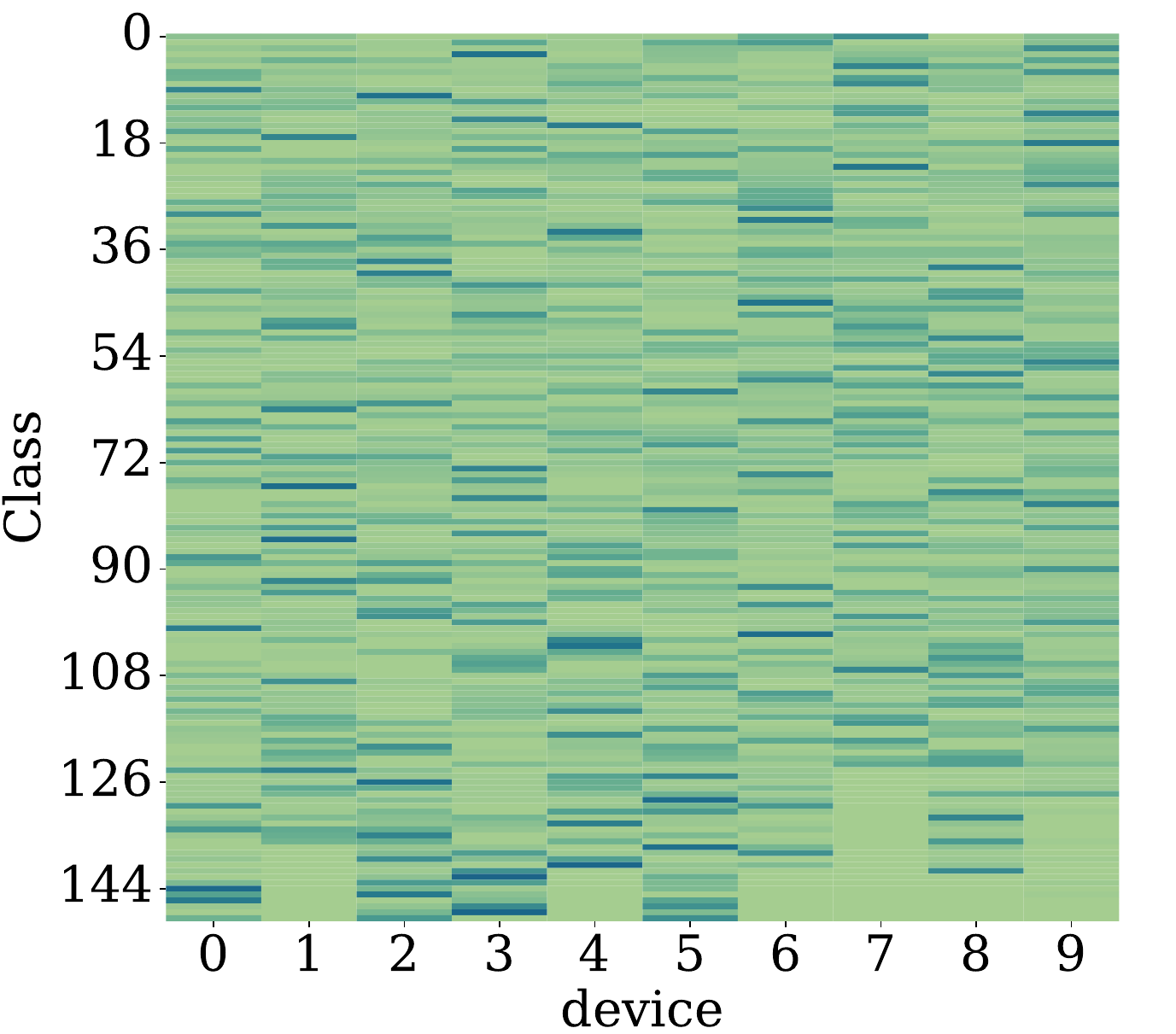}}
    \subfloat[\textit{p.c.c.d.} partition]
    {\includegraphics[width=1.265in]{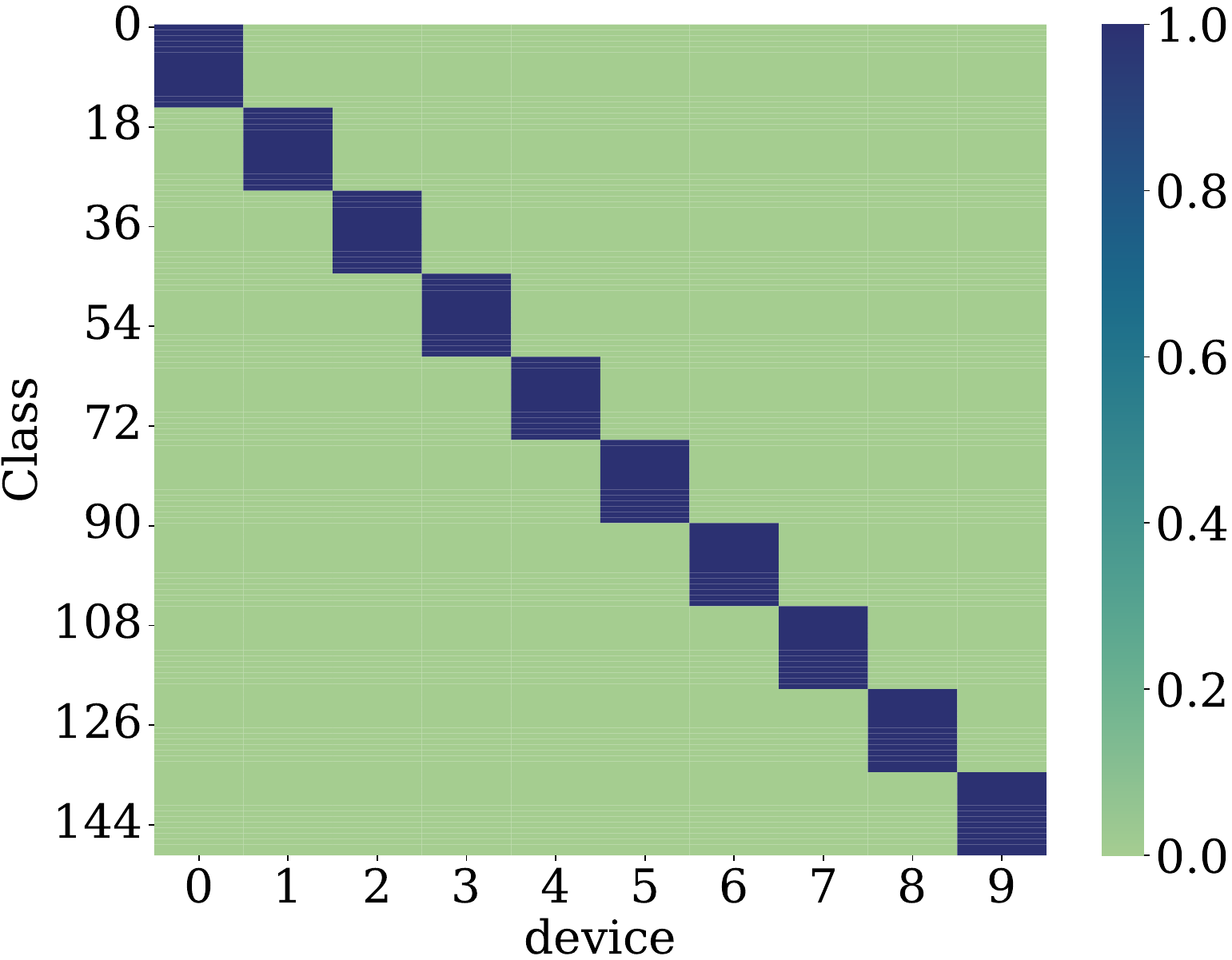}}
    \caption{Visualization of three different data distributions on CUB dataset. (a) For a \textit{i.i.d.} partition, we uniformly split the set of seen classes onto ten devices. (b) To generate \textit{non-i.i.d.} data, Dirichlet distribution with concentration parameter $\alpha = 0.5$ is used. (c) In \textit{p.c.c.d.}, each device owns non-overlapping classes.}
    \label{distribution_CUB}
\end{figure}


\begin{figure}[t]
    \centering
    \includegraphics[width=\linewidth]{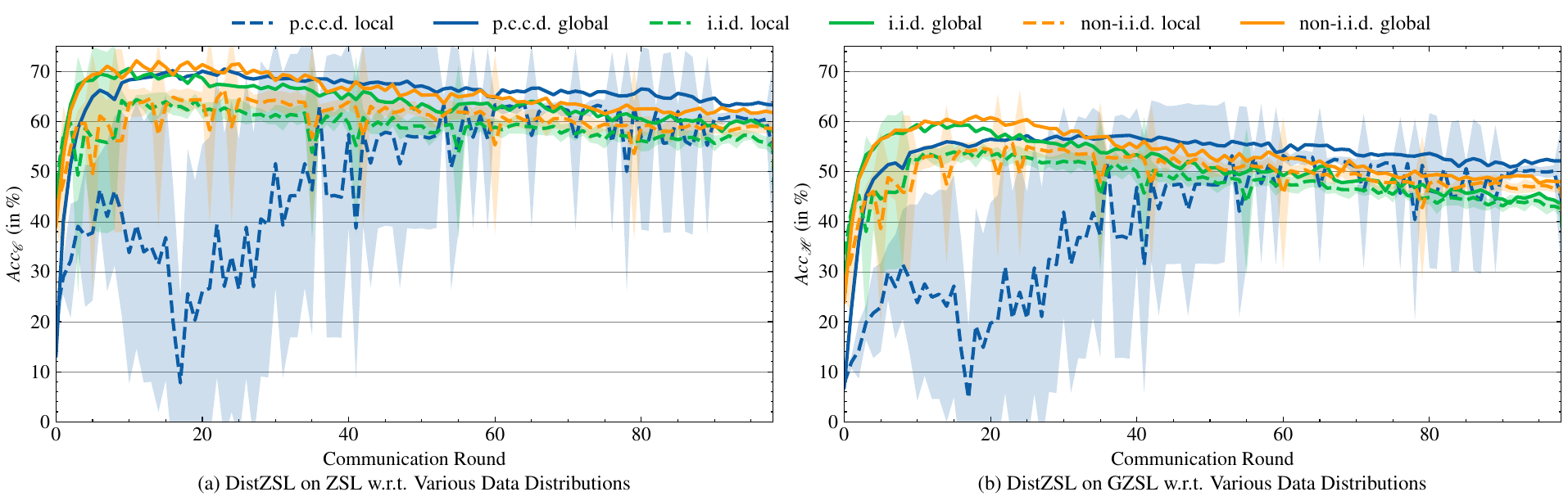}
    \caption{DistZSL training curves on ZSL and GZSL settings with different data distributions.}
    \label{fig:iid}
    \vspace{-10pt}
\end{figure}

\vspace{-5pt}
\subsection{Analyses and Discussions}
\subsubsection{Impact of FL Frameworks}
In order to understand the impact of FL frameworks on ZSL methods, we execute an extensive study combining three different ZSL methods with five representative FL frameworks. This results in the creation of 15 unique method combinations. We meticulously tuned the hyperparameters of these combined baselines to ensure optimal performance.
We conducted experiments on three different data distributions, the results of which are outlined in Table \ref{performance1}. We observe that the best performance results among baseline methods are consistently achieved by the combinations of APN+Scaffold and APN+FedProx on the CUB and AwA2 datasets. For the SUN datasets, GEM+FedAvg, GEM+Scaffold, and GEM+FedProx perform best.

These results confirm that \textit{non-i.i.d.} FL frameworks can enhance the ability to handle data heterogeneity issues within FL settings, even in the case of fine-grained datasets. Nevertheless, the performance on \textit{p.c.c.d.} data is significantly below the conventional \textit{non-i.i.d.} setting. In addition, we also noted certain failure cases. For instance, while the combination of APN+MOON achieved excellent results on the CUB and AwA2 datasets, it failed to converge on the SUN dataset. This suggests that while certain combinations may prove effective in some scenarios, they may not universally translate to all datasets or problem types.

\subsubsection{Impact of \textit{i.i.d.}, \textit{non-i.i.d.}, and \textit{p.c.c.d.} Distributions}
Our motivation is to learn a ZSL model with \textit{p.c.c.d} data distribution across devices. To investigate the impact of different data distributions, we conduct experiments with the data sampled from \textit{i.i.d.} and \textit{non-i.i.d.} distributions to compare with the model trained in our \textit{p.c.c.d.} setting. In \textit{i.i.d.}, we uniformly split the set of the seen classes onto ten devices, whereas in non-\textit{i.i.d.} we use a Dirichlet distribution $Dir(\alpha)$ with a concentration parameter $\alpha=0.5$ to sample the data partition. To facilitate a better understanding of the three data distributions, we visualize them on CUB in Figure \ref{distribution_CUB}. 

In Figure \ref{fig:iid}, we illustrate the learning curves with three different types of data distributions. The dashed lines and shaded areas represent the performance and its standard deviation  across local devices before the global aggregation; the solid lines indicate the global performance. As the devices in \textit{p.c.c.d.} are trained on non-overlapping classes, the performance variation is more significant. Figure \ref{fig:iid} also confirms that the \textit{non-i.i.d.} setting witnesses a slightly higher variation than the \textit{i.i.d.} setting, which is resulted from that the sample numbers of different classes vary across clients in the \textit{non-i.i.d.} protocol. For the global performance, the highest recognition accuracy of the proposed model trained in the challenging \textit{p.c.c.d.} setting is only slightly inferior to ones trained in the \textit{i.i.d.} and non-\textit{i.i.d.} settings. Moreover, the \textit{p.c.c.d.} setting tends to yield the best performance in both ZSL and GZSL setups, particularly as training proceeded over more communication rounds. This verifies the proposed DistZSL framework is agnostic to the different data distributions.

\subsubsection{Impact of Various Sampling Fractions}  \label{exp:vsf}

\begin{figure}[t]
    \centering
    \includegraphics[width=0.98\linewidth]{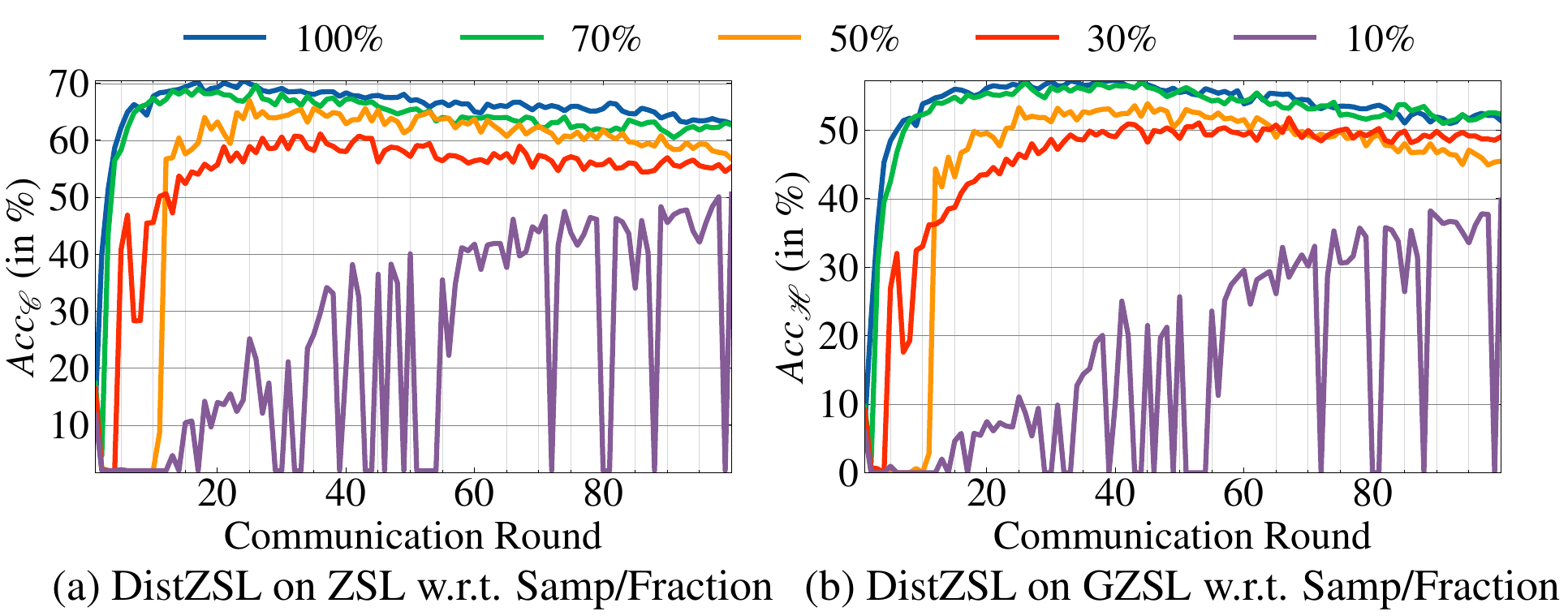}
    \caption{(a) and (b): DistZSL training curves on ZSL and GZSL settings w.r.t. various sampling fractions. }
    \label{fig:sampling}
    \vspace{-15pt}
\end{figure}

A common practice in FL is to sample a subset of clients during each communication round. The choice of this subset, or the \textit{sampling fraction}, can potentially have a significant impact on the efficiency and effectiveness of the learning process.
In order to delve deeper into the implications of different sampling fractions, we carry out a series of experiments. We vary the sampling fraction, choosing either 10\%, 30\%, 50\%, 70\%, or 100\% of participants to take part in each communication round. Our primary objective is to observe the effect that the sampling fraction had on the rate of convergence, a critical measure of the efficiency of the FL training process.

Our findings, illustrated in Figure \ref{fig:sampling}(\textcolor{red}{a}) and \ref{fig:sampling}(\textcolor{red}{b}), reveal some interesting insights. It appears that after 100 rounds of communication, the performance levels achieved by both ZSL and GZSL are comparable across all sampling fractions but 10\%. This suggests that even with smaller sampling fractions, performance levels could be maintained. However, there is a notable caveat to these findings. We observe that when the sampling fraction is reduced (\textit{i.e.,} fewer participants are included in each communication round), the rate of convergence is slower. Furthermore, performance exhibits increased instability, particularly when we attempt to aggregate local models that have been trained on subsets of seen classes. This observation highlights the importance of considering the sampling fraction in the design and implementation of federated learning systems. Although smaller fractions might still yield comparable performance levels, they may also introduce challenges in terms of slower convergence rates and less stable performance. Thus, striking a balance between sampling fraction and system efficiency and stability becomes a key consideration in the deployment of FL strategies.

\begin{figure}[t]
    \centering
    \includegraphics[width=0.97\linewidth]{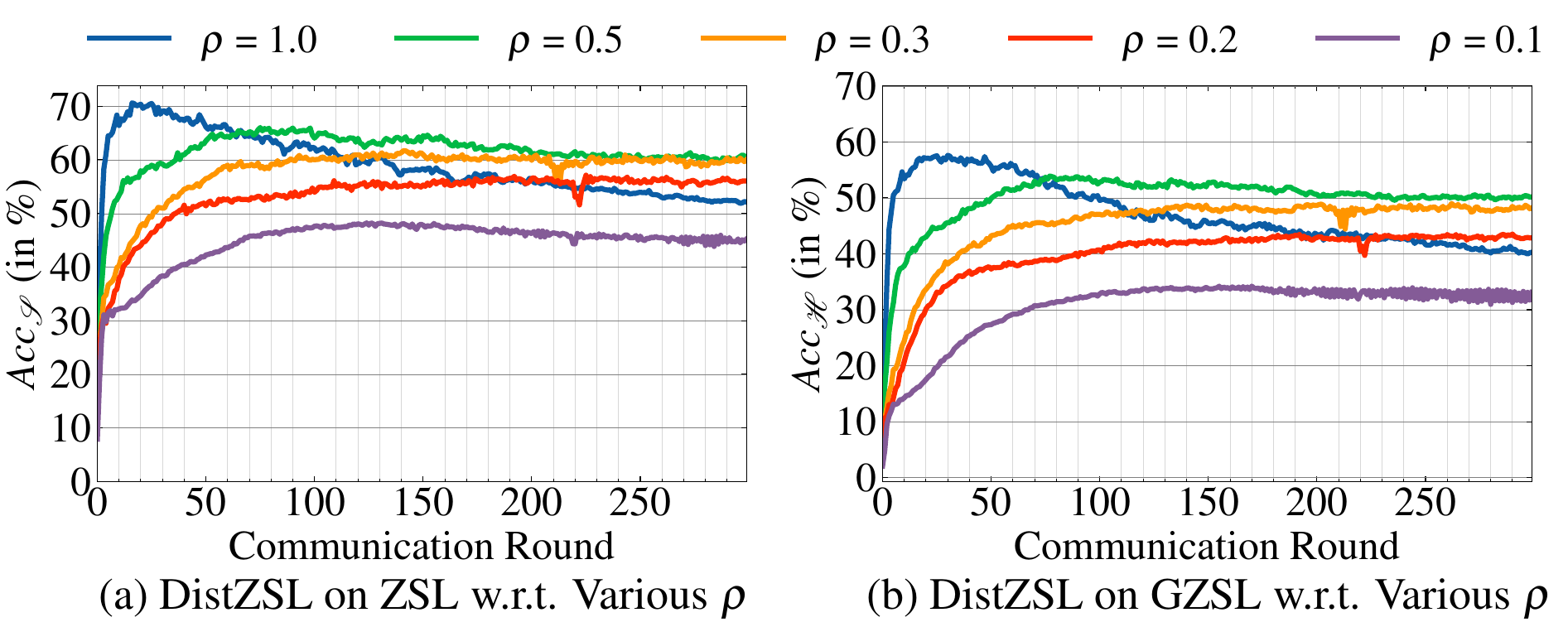}
    \vspace{-1pt}
    \caption{DistZSL training curves on the \textit{p.c.c.d.} setting w.r.t. partial data samples on the CUB dataset.}
    \label{fig:lesssamples}
\end{figure}

\begin {table}[t]
\caption {Effects of various client numbers and corresponding sampling fractions on CUB dataset of DistZSL.}
\begin{center}
\scalebox{0.9}{
\begin{tabular}[t]{c c c c ccc }
\toprule
 \#Clients & \#Sampled & \#Local Classes & Acc$_{\mathcal{C}}$ &  Acc$_{\mathcal{Y}^{u}}$ & Acc$_{\mathcal{Y}^{s}}$ & Acc$_{\mathcal{H}}$  \\
 \hline 
    10 & 3 & 15
    & 60.2          & 47.5              & 47.3           & 47.4
    \\
    10 & 5 & 15
    & 59.6          & 48.2              & 48.8           & 48.5
    \\
    10 & 10 & 15    
    & 71.6          & 57.5               & 58.0          & 57.8
    \\
    \hline 
    20 & 5 & 7/8
    & 60.3          & 45.7               & 48.6          & 47.1
    \\
    20 & 10  & 7/8
    & 62.4          & 45.7               & 52.5          & 48.9
    \\
    20 & 20      & 7/8
    & 62.5          & 48.9               & 49.0          & 48.9
    \\
    \hline 
    30 & 5 & 5
    & 57.5          & 43.7              & 40.7             & 42.1
    \\
    30 & 10 &5
    & 60.4          & 45.2              & 44.4            & 44.8
    \\
    30 & 30   & 5  
    & 61.0 &46.4 & 47.2 & 46.8
    \\
\hline\bottomrule
\end{tabular}}
\end{center}
\label{clientnumber}
\vspace{-15pt}
\end {table} 

\subsubsection{Impact of Client Number} 
To explore the impact of client numbers and corresponding sampling fractions, we conduct experiments with 10, 20 and 30 client numbers on the \textit{p.c.c.d.} setting. Partitioning the dataset into different numbers of clients results in different numbers of locally available classes. In addition, we sample various fractions of the clients in each communication round. All the experiments are conducted with the same set of hyper-parameters without particular adjustments. As shown in Table \ref{clientnumber}, when partitioning the dataset into more clients, the performance moderately drops. Also, similar to the conclusion in Section \ref{exp:vsf}, when sampling fewer clients in each communication round, the performance slightly decreases.

\subsubsection{Impact of Few-shot Samples} 
Understanding the impact of data quantity within local classes is crucial in machine learning scenarios, so we conducted a series of experiments to gain more insights into this. These experiments are performed on the CUB dataset, using a \textit{p.c.c.d.} setting.
In this setup, each client only had a subset of training samples from non-overlapping classes. To simplify the process, we defined a local data ratio, denoted as $\rho$, with values set to 0.1, 0.2, 0.3, 0.5, and 1.0. This ratio represents the percentage of training samples each client possesses. Figure \ref{fig:lesssamples} illustrates the detailed learning curves with various levels of $\rho$. Interestingly, our proposed method showed a robust performance even with reduced data quantities. For instance, when trained with only 10\% of data (approximately 6 samples per class in CUB), our method still managed to achieve around 45\% accuracy on unseen classes.

\subsubsection{Hyper-parameter Sensitivity}
Our proposed DistZSL model's overall objective function for local model training is controlled by three hyperparameters, namely the weights of $\ell_{kl}$, $\ell_{bc}$, and $\ell_{ad}$. To gain a deeper understanding of how these various components influence the effectiveness of the proposed DistZSL model, we have conducted an examination of the sensitivity of these three hyper-parameters. The results of this analysis are visually represented in Figure \ref{fig:hyper}. As can be observed from the figure, the performance on the CUB dataset reaches its optimum when the weight of $\ell_{kl}$ is configured to 10, and $\ell_{bc}$ and $\ell_{ad}$ are set to 0.1 and 0.3, respectively. This signifies that the balance among these three components plays a crucial role in achieving optimal performance for the proposed DistZSL model.

\begin{figure}
    \centering
    \includegraphics[width=1\linewidth]{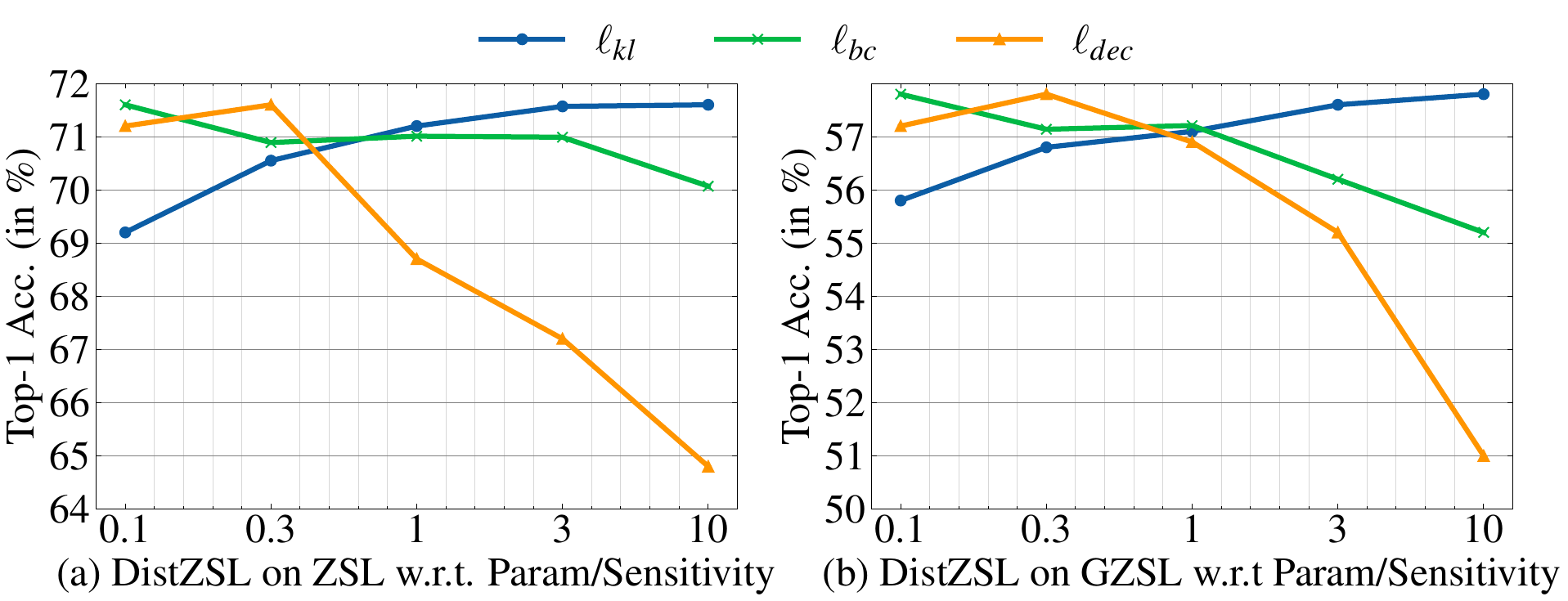}
    \caption{DistZSL hyper-parameter sensitivity on ZSL and GZSL settings.}
    \label{fig:hyper}
    \vspace{-10pt}
\end{figure}

\section{Conclusion}
In this paper, we explore the concept of Distributed Zero-Shot Learning ({DistZSL}). We propose a solution that incorporates attribute-based learning, a bilateral visual-semantic connection, and a cross-device attribute regularizer to harmonize visual-semantic predictions across various devices. This proposed methodology proves capable of managing diverse data distributions, especially partial class-conditional data (\textit{p.c.c.d.})—a challenging aspect for all existing \textit{non-i.i.d.} FL methods. We set a benchmark for DistZSL by integrating state-of-the-art ZSL methods with \textit{non-i.i.d.} FL frameworks, leading to an in-depth evaluation and comparison of the resulting performance metrics. Furthermore, our empirical analysis indicates that the use of attribute-based learning can significantly mitigate the global aggregation difficulties typically associated with traditional attribute-free learning. Through extensive experimentation, we have validated that the proposed approach is equipped to manage the dynamics of participant engagement, including various sampling fractions, client numbers, and partial data samples. 

\balance
\bibliographystyle{IEEEtran}
\bibliography{Bib} 

\vfill

\newpage

\appendices
\section{Baselines Methods}
We provide a comprehensive introduction to the baseline ZSL and FL methods as follows.

 {\textbf{APN}} \cite{xu2020attribute} leverages the power of local and global features for image understanding. It employs the attribute prototype network to decipher the local features of individual images while a visual semantic embedding layer is harnessed to learn global features. A fundamental aspect of attribute prototypes is their ability to convert feature maps into attribute maps.  This transition diminishes the channel dimension of feature maps to align with the number of predefined attributes associated with each class. Each attribute map is designed such that the most prominent value signifies the location of a distinct semantic meaning, thereby providing a spatial context for that semantic within the image.

 {\textbf{GEM}} \cite{liu2021goal} models the human cognitive process for identifying unseen classes in zero-shot learning scenarios. The idea behind GEM is to use the attribute description of objects as a guide for predicting human gaze patterns. The predicted gaze information is then used to construct attribute attention maps. These attention maps provide valuable insights into the visual features that the model should focus on during object recognition tasks. This is particularly effective when identifying objects that belong to unseen classes, mirroring the human visual perception of unfamiliar scenarios.

 {\textbf{MSDN}} \cite{chen2022msdn} encompasses two interconnected attention sub-networks: the attribute-to-visual attention and the visual-to-attribute attention sub-nets. The attribute-to-visual attention sub-net learns to highlight visual features based on the given attribute information. On the other hand, the visual-to-attribute attention sub-net generates attention for attributes based on the learned visual features. To foster collaboration and reciprocal guidance between these sub-networks, a unique semantic distillation loss is incorporated. This loss quantifies the divergence between the two sub-networks' outputs, encouraging their alignment. Therefore, MSDN utilizes a bidirectional attention mechanism that integrates mutual guidance into the learning process, resulting in a more thorough and robust comprehension of unseen classes in ZSL.

{\textbf{SVIP}} \cite{chen2025svip} introduces a semantically contextualized visual patch framework that addresses the problem of semantic–unrelated visual information in zero-shot learning. Instead of attempting to suppress irrelevant details after feature extraction, SVIP operates directly at the patch level by identifying and handling non-semantic patches before they propagate through the network. A self-supervised patch selection strategy aggregates attention maps from the transformer backbone to estimate semantic relevance, and a lightweight patch classifier is trained to detect semantic–unrelated patches. Rather than discarding these patches, SVIP replaces them with learnable patch embeddings that are initialized from semantic descriptors, thereby preserving structural consistency and injecting semantic cues. Furthermore, an attribute localization component leverages these contextualized patches to enhance the discriminability of attributes. This patch-level intervention leads to stronger semantic alignment and state-of-the-art performance on standard ZSL and GZSL benchmarks.

 {\textbf{FedAvg}} \cite{mcmahan2017communication} has become a benchmark for FL methods due to its efficient and straightforward approach. Initially, the server sends the global model to randomly selected parties. These parties then use their local datasets to update the model. After the local models have been updated, they are returned to the server. The server concludes the round by averaging the received local models to update the global model. This method varies from traditional distributed SGD (FedSGD) because it allows parties to update their local model across multiple epochs, reducing the number of communication rounds and making the process more communication-efficient.

 {\textbf{FedProx}} \cite{li2020federated} enhances the FedAvg approach by introducing an improvement to the local objective. This method limits the degree of divergence each local model can have from the global model by constraining on the extent of local updates. This is achieved by adding an L2 regularization term to the local objective function, which sets a boundary on the distance between the local model and the global model. The underlying concept of this approach is to ensure that the averaged model, procured after aggregating all local updates, is not vastly divergent from the global optimum. The impact of the L2 regularization is managed by a hyper-parameter, controlling its influence.

 {\textbf{FedNova}} \cite{wang2020tackling} is an enhanced method based on FedAvg, particularly focusing on the model aggregation stage. It recognizes that distinct clients or parties may execute varying numbers of local updates or steps in each round, due to factors like differences in computational resources, time limitations, or local dataset size. The core idea behind FedNova is that parties executing a greater number of local steps will likely produce larger local updates. To ensure a balanced global model update, FedNova adopts a process where local updates from each party are normalized and scaled according to their respective number of local steps, prior to updating the global model. 

 {\textbf{Scaffold}} \cite{karimireddy2020scaffold} is a federated learning algorithm that considers the \textit{non-i.i.d.} nature of client data as a source of variance among the clients and utilizes variance reduction techniques to manage it. It introduces control variates for the server and each client, which are used to estimate the direction of model updates for the server and each client. The discrepancy between these two directions of updates is taken as an approximation of the drift in local training. As a result, SCAFFOLD adjusts the local updates by accounting for this drift in local training.

 {\textbf{MOON}} \cite{li2021model}  employs the model similarity as a key strategy to optimize local training for each participant. In traditional FL, divergence of client-specific models can become significant due to the \textit{non-i.i.d.} characteristics of local data. MOON counters this divergence by promoting similarity across model representations for different clients. This is achieved through contrastive learning, a technique that draws similar instances closer and pushes dissimilar ones further apart. By enforcing consistency across local models, MOON facilitates more efficient learning and enhances the overall performance of the FL system. 

{\textbf{FedGloss}} \cite{caldarola2025beyond} addresses the limitations of sharpness-aware minimization in federated learning by focusing on global rather than local sharpness. While prior methods such as FedSAM apply SAM locally at each client, they suffer from the mismatch between local and global loss landscapes, meaning that flatter local minima do not always translate into global flatness. FedGloSS shifts the sharpness-aware optimization to the server side, directly targeting global flatness. To keep the method communication-efficient, FedGloSS avoids extra client-server exchanges by approximating sharpness using the previous round’s pseudo-gradients, thereby eliminating the need for additional forward–backward passes on clients. This design reduces client computation, maintains communication efficiency, and consistently achieves flatter minima, yielding better generalization across heterogeneous federated vision benchmarks.

\section{Theoretical Analysis}

We provide theoretical support for the two key components in DistZSL, including the cross-node attribute regularizer $\ell_{\mathrm{kl}}$ (Eq.~7) and the global attribute-to-visual consensus $\ell_{\mathrm{bc}}$ (Eq.~8). Throughout, classes have attribute prototypes $\mathcal{A}=\{\va_y\}_{y\in\mathcal{Y}}$, all probability vectors lie in the simplex $\Delta^{|\mathcal{Y}|-1}$, and softmax temperature $\tau>0$ is fixed.

\subsubsection{Setup and assumptions.}
Let $f:\mathcal{X}\!\to\!\mathbb{R}^{d_v}$ be the backbone, $g:\mathbb{R}^{d_v}\!\to\!\mathbb{R}^{d_a}$ the attribute regressor, and $h:\mathbb{R}^{d_a}\!\to\!\mathbb{R}^{d_v}$ the semantic-to-visual regressor. We define $w$ as the model parameters.
For a sample $(\vx,y)$ on client $k$, define logits $z_k(\vx)=\widehat \va_k(\vx)^{\top}A \in \mathbb{R}^{|\mathcal{Y}|}$ with $\widehat \va_k(\vx)=g(f(\vx))$, and the corresponding client distribution
\begin{equation}
p_k(\cdot\mid \vx;\tau)=\mathrm{softmax}\!\big(\vz_k(\vx)/\tau\big)\in\Delta^{|\mathcal{Y}|-1}.
\end{equation}
Let $\Gamma\in\mathbb{R}^{|\mathcal{Y}|\times|\mathcal{Y}|}$ denote the global semantic similarity matrix (estimated once on the server), and $p_\Gamma(\cdot\mid y;\tau)=\mathrm{softmax}(\Gamma_y/\tau)$ denote the target distribution for class $y$.

We make the following mild assumptions restricted to the data manifold $\mathcal{M}\subset\mathcal{X}$ in distributed learning setting.

\begin{itemize}
\item[A1] (Bi-Lipschitz decoder locally on $\mathrm{Im}(g\!\circ\! f)$). There exist constants $0<c_h\le L_h<\infty$ such that for all $\va_1,\va_2$ in a neighborhood of $\mathrm{Im}(g\!\circ\! f)$,
\begin{align}
c_h\|\va_1-\va_2\|\ &\le\ \|h(\va_1)-h(\va_2)\|\ \nonumber \\ &\le\ L_h\|\va_1-\va_2\|.
\end{align}
\item[A2] (Bounded reconstruction). Training with $\ell_{\mathrm{bc}}$ yields a uniform bound $\|h(g(f(\vx))) - f(\vx)\|\le \delta$ for all $\vx\in\mathcal{M}$ and some $\delta\ge 0$.
\item[A3] (Model smoothness near FedAvg iterate). For a fixed $\vx$, the mapping $w\mapsto z(\vx;w)$ (logits under parameters $w$) is $L_z$-Lipschitz in a neighborhood of the aggregated parameters $\bar w$, and softmax has Lipschitz constant $L_{\mathrm{sm}}(\tau)$ in logits, such that 
$\|z(\vx;w_1) - z(\vx;w_2)\| \le L_z\|w_1-w_2\|, \forall w_1, w_2 \in \mathcal{N}(w)$
and
$\|\softmax(\frac{z_1}{\tau}) - \softmax(\frac{z_2}{\tau})\| \le L_{\mathrm{sm}}(\tau)\|z_1-z_2\|$.
\item[A4] (Prototype separability). Prototypes are unit-normalized, $\|\va_y\|=1$, and have attribute margin $\Delta_y=\min_{y'\neq y}\|\va_y-\va_{y'}\|>0$.
\end{itemize}

\subsubsection{Cross-node attribute regularization}

Client $k$ minimizes the KL divergence to the global target
\begin{equation}
\ell_{\mathrm{kl}}^{(k)}(\vx,y)=\tau^2\,\mathrm{KL}\!\big(p_\Gamma(\cdot\mid y;\tau)\,\Vert\,p_k(\cdot\mid \vx;\tau)\big).
\end{equation}

\begin{lemma}[Client-level alignment]
\label{lem:client_alignment}
If $\mathbb{E}_{(\vx,y)}[\ell_{\mathrm{kl}}^{(k)}(\vx,y)]\le \varepsilon_k$ for some $\varepsilon_k > 0$ for client $k$, then for almost all $(\vx,y)$
\begin{equation}
\big\|p_k(\cdot\mid \vx;\tau)-p_\Gamma(\cdot\mid y;\tau)\big\|_1
\ \le\ \sqrt{\tfrac{2}{\tau^2}\,\varepsilon_k}.
\end{equation}
where $\varepsilon_k$ denotes the expected cross-node alignment error of client $k$, i.e.,
$\varepsilon_k = \mathbb{E}_{(\vx,y)}[\ell^{(k)}_{\mathrm{kl}}(\vx,y)]$.
Consequently, for any two clients $j,k$,
\begin{equation}
\big\|p_j(\cdot\mid \vx;\tau)-p_k(\cdot\mid \vx;\tau)\big\|_1
\ \le\ \sqrt{\tfrac{2}{\tau^2}\,\varepsilon_j}+\sqrt{\tfrac{2}{\tau^2}\,\varepsilon_k}.
\end{equation}
\end{lemma}

\begin{proof}
By definition of the KL-based regularization loss, we take expectation over $(\vx,y)$ yields
\begin{equation}
\mathbb{E}_{(\vx,y)}\!\big[\ell^{(k)}_{\mathrm{kl}}(\vx,y)\big]
= \tau^2 \,\mathbb{E}_{(\vx,y)}\!\left[\mathrm{KL}(p_\Gamma \,\Vert\, p_k)\right]
\;\leq\; \varepsilon_k.
\end{equation}

Next, according to Pinsker’s inequality, for any distributions $p,q$, we have $\|p-q\|_1 \;\leq\; \sqrt{2 \,\mathrm{KL}(p\Vert q)}$. Applying this to $p=p_\Gamma(\cdot\mid y;\tau)$ and $q=p_k(\cdot\mid \vx;\tau)$ gives
\begin{equation}
\|p_\Gamma - p_k\|_1 \;\leq\; \sqrt{2\,\mathrm{KL}(p_\Gamma\Vert p_k)}
\;=\; \sqrt{\tfrac{2}{\tau^2}\,\ell^{(k)}_{\mathrm{kl}}(\vx,y)}.
\end{equation}

Now take expectation over $(\vx,y)$. Since the square root is concave, Jensen’s inequality gives
\begin{align}
& \mathbb{E}_{(\vx,y)}\!\big[\|p_\Gamma - p_k\|_1\big] \nonumber\\
&\;\leq\; \sqrt{\tfrac{2}{\tau^2}\,\mathbb{E}_{(\vx,y)}[\ell^{(k)}_{\mathrm{kl}}(\vx,y)]}
\;\leq\; \sqrt{\tfrac{2}{\tau^2}\,\varepsilon_k}.
\end{align}

Finally, for two clients $j,k$, the triangle inequality yields
\begin{equation}
\|p_j - p_k\|_1 \;\leq\; \|p_j - p_\Gamma\|_1 + \|p_k - p_\Gamma\|_1.
\end{equation}
Taking expectations and applying the bounds above completes the proof.
\end{proof}

\begin{theorem}[Server-level guarantee under FedAvg]
\label{thm:server_alignment}
Let $\bar p(\cdot\mid \vx;\tau)=\sum_{k}\alpha_k\,p_k(\cdot\mid \vx;\tau)$ be the mixture of client distributions with FedAvg weights $\alpha_k=\frac{|{\cal D}_{s,k}|}{\sum_j |{\cal D}_{s,j}|}$. Then
\begin{align}
&\mathrm{KL}\!\big(\bar p(\cdot\mid \vx;\tau)\,\Vert\,p_\Gamma(\cdot\mid y;\tau)\big)\   \nonumber \\
& \le\ \sum_k \alpha_k\,\mathrm{KL}\!\big(p_k(\cdot\mid \vx;\tau)\,\Vert\,p_\Gamma(\cdot\mid y;\tau)\big).
\end{align}
If assumption A3 holds and the global model distribution $p(\cdot\mid x;\bar w,\tau)$ is within $\xi$, in $L1$of $\bar p(\cdot\mid x;\tau)$, then
\begin{align}
& \mathrm{KL}\!\big(p(\cdot\mid \vx;\bar w,\tau)\,\Vert\,p_\Gamma(\cdot\mid y;\tau)\big)
\ \nonumber \\ &\le\ \sum_k \alpha_k\,\mathrm{KL}\!\big(p_k\,\Vert\,p_\Gamma\big)\ +\ C\,\xi,
\end{align}
for a constant $C$ depending only on $\tau$.
\end{theorem}

\begin{proof}
For any fixed $q$ and distributions $\{p_k\}$ with weights $\{\alpha_k\}$,
KL is convex in its first argument:
\begin{align}
&\mathrm{KL}\!\Big(\sum_k \alpha_k p_k \,\Big\Vert\, q\Big) 
= \sum_i \Big(\sum_k \alpha_k p_{k,i}\Big)\log\frac{\sum_k \alpha_k p_{k,i}}{q_i} \nonumber \\
& \;\le\; \sum_k \alpha_k \sum_i p_{k,i}\log\frac{p_{k,i}}{q_i}
= \sum_k \alpha_k\,\mathrm{KL}(p_k\Vert q),
\end{align}
where the inequality follows from Jensen applied coordinate-wise to
$u\mapsto u\log(u/q_i)$, which is convex on $u>0$.
Setting $q=p_\Gamma(\cdot\mid y;\tau)$ and $\sum_k\alpha_k p_k=\bar p(\cdot\mid x;\tau)$ gives
\begin{equation}
\mathrm{KL}\!\big(\bar p \,\Vert\, p_\Gamma\big)\ \le\ \sum_k \alpha_k\,\mathrm{KL}\!\big(p_k \,\Vert\, p_\Gamma\big).
\label{A}
\end{equation}
By A3, for a fixed $\vx$ the logit map $w\mapsto z(x;w)$ is $L_z$-Lipschitz near $\bar w$, and softmax with temperature $\tau$ is $L_{\mathrm{sm}}(\tau)$-Lipschitz in logits. Hence, for any client $k$,
\begin{align}
&\big\|p(\cdot\mid \vx;\bar w,\tau)-p_k(\cdot\mid \vx;\tau)\big\|_1 \nonumber \\
& \ \le\ L_{\mathrm{sm}}(\tau)\,\big\|z(\vx;\bar w)-z(\vx;w_k)\big\|\  \\
 & \le\ L_{\mathrm{sm}}(\tau)\,L_z\,\big\|\bar w-w_k\big\|. \nonumber
\end{align}
Using convexity of the $\ell_1$ norm and $\bar p=\sum_k \alpha_k p_k$,
\begin{align}
& \big\|p(\cdot\mid \vx;\bar w,\tau)-\bar p(\cdot\mid \vx;\tau)\big\|_1 \nonumber \\
& \ \le\ \sum_k \alpha_k \big\|p(\cdot\mid \vx;\bar w,\tau)-p_k(\cdot\mid \vx;\tau)\big\|_1 \\
& \ \le\ L_{\mathrm{sm}}(\tau)\,L_z \sum_k \alpha_k \big\|\bar w-w_k\big\|. \nonumber
\end{align}
Denote the right-hand side by $\xi$ for brevity. This is the explicit form used in the theorem.

Let $F(u)=\sum_i u_i\log(u_i/q_i)$ with $q=p_\Gamma(\cdot\mid y;\tau)$ fixed.
If all coordinates of $u$ and $q$ are bounded below:
$u_i\ge m_p(\tau)>0$, $q_i\ge m_\Gamma(\tau)>0$,
then by the mean-value theorem,
\begin{align}
\big|F(r)-F(s)\big| & \;=\; \big|\nabla F(\tilde u)^\top (r-s)\big|
\ \nonumber \\ &\le\ \|\nabla F(\tilde u)\|_\infty\,\|r-s\|_1,
\end{align}
for some $\tilde u$ on the segment $[r,s]$.
Since 
\begin{align}
& \nabla F(u)_i=\log(u_i/q_i)+1,
\|\nabla F(\tilde u)\|_\infty \le C(\tau) \nonumber \\ &:=\max_i\{\,|\log(m_p(\tau)/m_\Gamma(\tau))|+1\,\},
\end{align}
and therefore
\begin{equation}
\big|\mathrm{KL}(r\Vert q)-\mathrm{KL}(s\Vert q)\big| \;\le\; C(\tau)\,\|r-s\|_1.
\end{equation}
Apply this with $r=p(\cdot\mid \vx;\bar w,\tau)$ and $s=\bar p(\cdot\mid \vx;\tau)$ to obtain
\begin{align}
&\mathrm{KL}\!\big(p(\cdot\mid \vx;\bar w,\tau)\,\Vert\,p_\Gamma\big) \nonumber \\
& \ \le\ \mathrm{KL}\!\big(\bar p(\cdot\mid \vx;\tau)\,\Vert\,p_\Gamma\big) \;+\; C(\tau)\,\xi.
\label{B}
\end{align}
From Eq. \ref{A} and \ref{B}, we have
\begin{align}
& \mathrm{KL}\!\big(p(\cdot\mid \vx;\bar w,\tau)\,\Vert\,p_\Gamma\big) \nonumber \\
& \ \le\ \sum_k \alpha_k\,\mathrm{KL}\!\big(p_k \,\Vert\, p_\Gamma\big)\;+\;C(\tau)\,\xi,
\end{align}
which completes the proof.
\end{proof}

Theorem~\ref{thm:server_alignment} states that, as each client reduces its local $\ell_{\mathrm{kl}}$, the global model’s predictive distribution moves monotonically closer to the target semantic distribution $p_\Gamma$, up to the small averaging approximation. Hence, it aligns attribute similarity patterns across clients.

\subsubsection{Global Attribute-to-Visual Consensus}

The bilateral loss
\begin{equation}
\ell_{\mathrm{bc}}(\vx)=\|h(g(f(\vx)))-f(\vx)\|^2
\end{equation}
enforces that $h$ acts as an approximate left-inverse of $g\!\circ\! f$ on the data manifold.
\\
\begin{lemma}[Information preservation via approximate left-inverse]
\label{lem:bilipschitz_lower}
Under A1–A2, for any $\vx_1,\vx_2\in\mathcal{M}$,
\begin{equation}
\|g(f(\vx_1))-g(f(\vx_2))\|
\ \ge\ \tfrac{1}{L_h}\,\|f(\vx_1)-f(\vx_2)\|\ -\ \tfrac{2\delta}{L_h}.
\end{equation}
\end{lemma}

\begin{proof}
Start from the triangle inequality by adding and subtracting the reconstructions:
\begin{align}
&\|f(\vx_1)-f(\vx_2)\| \nonumber \\
&= \big\| \big(f(\vx_1)-h(g(f(\vx_1)))\big)  + \big(h(g(f(\vx_1)))-h(g(f(\vx_2)))\big) \nonumber \\
& \hspace{4.1cm} + \big(h(g(f(\vx_2)))-f(\vx_2)\big) \big\| \nonumber \\
&\le \|f(\vx_1)-h(g(f(\vx_1)))\|
   + \|h(g(f(\vx_1)))-h(g(f(\vx_2)))\| \nonumber \\
&\hspace{3.9cm}  + \|h(g(f(\vx_2)))-f(\vx_2)\|.
\end{align}

By A2, the first and third terms are each bounded by $\delta$:
\begin{equation}
\|f(\vx_i)-h(g(f(\vx_i)))\|\ \le\ \delta,\quad i\in\{1,2\}.
\end{equation}
Apply the upper Lipschitz bound from A1 to the middle term (with $a_1=g(f(\vx_1))$, $a_2=g(f(\vx_2))$):
\begin{align}
& \|h(g(f(\vx_1)))-h(g(f(\vx_2)))\|\ \nonumber \\
& \le\ L_h\,\|g(f(\vx_1))-g(f(x_2))\|.
\end{align}
Combine the bounds to obtain
\begin{equation}
\|f(\vx_1)-f(\vx_2)\|\ \le\ 2\delta + L_h\,\|g(f(\vx_1))-g(f(\vx_2))\|.
\end{equation}
Finally, rearrange:
\begin{equation}
\|g(f(\vx_1))-g(f(\vx_2))\|\ \ge\ \frac{1}{L_h}\,\|f(\vx_1)-f(\vx_2)\| - \frac{2\delta}{L_h}.
\end{equation}
\end{proof}

\paragraph{Interpretation.}
Lemma 3 states that distances in the visual space cannot collapse under $g$ (up to a $2\delta$ slack) because the decoder $h$ approximately inverts $g$ on the image of $f$: enforcing $\ell_{\mathrm{bc}}(\vx)=\|h(g(f(\vx)))-f(\vx)\|^2$ small (small $\delta$) guarantees that attribute predictions $g(f(\vx))$ retain discriminative information from $f(\vx)$.


\begin{lemma}[Attribute error bound from reconstruction]
\label{lem:attr_error_bound}
Fix $(\vx,y)$ and assume A1–A2. Then
\begin{equation}
\|g(f(\vx)) - \va_y\|\ \le\ \tfrac{1}{c_h}\,\big(\, \|h(\va_y)-f(\vx)\| + \delta\,\big).
\end{equation}
In particular, if $h(\va_y)$ approximates the class center in visual space with error $\varepsilon_y=\|h(\va_y)-f(\vx)\|$, then $\|g(f(\vx))-\va_y\|\le (\varepsilon_y+\delta)/c_h$.
\end{lemma}

\begin{proof}
Assumption A1 states that for all $\va_1,\va_2$ in a neighborhood of ${\rm Im}(g\circ\! f)$,
$c_h\|\va_1-\va_2\| \le \|h(\va_1)-h(\va_2)\|$.
Choose $a_1=g(f(\vx))$ and $\va_2=\va_y$ to obtain
\begin{equation}
c_h\,\|g(f(\vx))-\va_y\| \;\le\; \|\,h(g(f(\vx))) - h(\va_y)\,\|.
\label{eq:step1}
\end{equation}
\\
Add and subtract $f(\vx)$ inside the norm and use the triangle inequality:
\begin{align}
&\|\,h(g(f(\vx))) - h(\va_y)\,\| \nonumber \\
&= \|\,\big(h(g(f(\vx))) - f(\vx)\big) + \big(f(\vx)-h(\va_y)\big)\,\| \nonumber\\
&\le \|\,h(g(f(\vx))) - f(\vx)\,\| + \|\,f(\vx)-h(\va_y)\,\|.
\label{eq:step2}
\end{align}

By A2, $\|\,h(g(f(\vx))) - f(\vx)\,\| \le \delta$. Plugging this into \eqref{eq:step2} yields
\begin{equation}
\label{eq:step3}
\|\,h(g(f(\vx))) - h(\va_y)\,\| \;\le\; \delta + \|\,h(\va_y)-f(\vx)\,\|.
\end{equation}

Combine \eqref{eq:step1} and \eqref{eq:step3}:
\begin{equation}
c_h\,\|g(f(\vx))-\va_y\| \;\le\; \delta + \|\,h(\va_y)-f(\vx)\,\|.
\end{equation}
Divide both sides by $c_h$ to obtain the claimed bound:
\begin{equation}
\|g(f(\vx)) - \va_y\| \;\le\; \frac{1}{c_h}\Big(\delta + \|h(\va_y)-f(\vx)\|\Big).
\end{equation}
Setting $\varepsilon_y=\|h(\va_y)-f(\vx)\|$ gives $\|g(f(\vx))-\va_y\|\le (\varepsilon_y+\delta)/c_h$.
\end{proof}

\begin{theorem}[Margin preservation for attribute-based classification]
\label{thm:margin}
Assume A1, A2, A4 and let $\varepsilon_y=\|h(\va_y)-f(\vx)\|$. If
\begin{equation}
\delta + \varepsilon_y\ <\ \tfrac{c_h}{2}\,\frac{\Delta_y^2}{\max_{y'\neq y}\|\va_y-\va_{y'}\|},
\end{equation}
then the attribute-based classifier using logits $s_{y'}=g(f(\vx))^{\top}\va_{y'}$ predicts the correct label $y$.
\end{theorem}

\begin{proof}
Let $\hat \va = g(f(\vx))$ and denote $d_{y'} \triangleq \|\va_y-\va_{y'}\|$ for $y'\neq y$. 
By A4, both $\va_y$ and $\va_{y'}$ are unit vectors, and we define 
$\Delta_y = \min_{y'\neq y} d_{y'}$ and $d_{\max} = \max_{y'\neq y} d_{y'}$.

We begin by decomposing the score difference between the correct class and any competitor:
\begin{align}
s_y - s_{y'} &= \hat \va^\top \va_y - \hat \va^\top \va_{y'} 
\nonumber \\ 
&= \va_y^\top(\va_y-\va_{y'}) + (\hat \va - \va_y)^\top(\va_y-\va_{y'}).
\end{align}

The first term can be simplified using the fact that both $\va_y$ and $\va_{y'}$ have unit norm. 
Specifically,
\begin{equation}
\va_y^\top(\va_y-\va_{y'}) = 1 - \va_y^\top \va_{y'} 
= \tfrac{1}{2}\|\va_y-\va_{y'}\|^2 
= \tfrac{1}{2} d_{y'}^{\,2}.
\end{equation}

The second term can be lower bounded by the Cauchy–Schwarz inequality:
\begin{align}
(\hat \va-\va_y)^\top(\va_y-\va_{y'}) & \ge -\,\|\hat \va-\va_y\|\,\|\va_y-\va_{y'}\| \nonumber \\
& = -\,\|\hat \va-\va_y\|\, d_{y'} .
\end{align}

Combining these results, we obtain for every $y'\neq y$,
\begin{equation}
s_y - s_{y'} \;\ge\; \tfrac{1}{2} d_{y'}^{\,2} - \|\hat \va-\va_y\|\, d_{y'}.
\end{equation}

Taking the minimum over all $y'\neq y$ shows that
\begin{equation}
\min_{y'\neq y}(s_y-s_{y'}) \;\ge\; \tfrac{1}{2}\Delta_y^{2} - \|\hat \va-\va_y\|\, d_{\max}.
\end{equation}
Thus, a sufficient condition for correct classification is
\begin{equation}
\|\hat \va-\va_y\| < \tfrac{1}{2}\,\frac{\Delta_y^2}{d_{\max}}.
\end{equation}

Finally, Lemma~\ref{lem:attr_error_bound} provides the bound 
\begin{equation}
\|\hat \va-\va_y\| \le \tfrac{1}{c_h}(\varepsilon_y+\delta).
\end{equation}
Substituting this into the sufficient condition above yields
\begin{equation}
\delta+\varepsilon_y < \tfrac{c_h}{2}\,\frac{\Delta_y^2}{d_{\max}}.
\end{equation}
Under this condition we have $s_y > s_{y'}$ for all $y'\neq y$, 
so the classifier assigns the correct label $y$.
\end{proof}

Theorem~\ref{thm:margin} shows that minimizing $\ell_{\mathrm{bc}}$ (small $\delta$) controls the deviation of predicted attributes from their class anchors, which in turn guarantees class-wise separation in the attribute-based classifier as long as prototypes are reasonably separated. Combined with Lemma~\ref{lem:bilipschitz_lower}, the bilateral connection prevents information loss from $f$ to $g(f(x))$ and stabilizes cross-device learning by keeping discriminative structure intact.

\section{Discussion on Privacy Preservation and Potential Risks}DistZSL inherits the standard privacy benefits of federated learning: raw data never leaves local devices, and only model updates are transmitted to the server. The additional components (cross-node attribute regularization and global attribute-to-visual consensus) rely solely on attribute prototypes and similarity matrices that are shared once across clients; these are dataset-level semantic statistics that do not expose individual examples.

Potential risks include model inversion or membership inference attacks based on shared model parameters, which are well-known challenges in FL in general. Importantly, DistZSL does not introduce new risks beyond existing FL methods. Furthermore, DistZSL is fully compatible with advanced privacy-preserving techniques such as secure aggregation, differential privacy, and homomorphic encryption, which can be adopted in future work to further enhance protection.

\section{Comprehensive Performance Comparison Between All FL Baselines}
In Table~\ref{performance2} and \ref{performance3}, we present the comprehensive tables that compare all six FL baseline methods.

\begin {table*}[!h]
\caption{Performance comparisons (\%) on five datasets among FL baselines, ZSL baselines, and the proposed DistZSL in centralized and \textit{i.i.d.} settings. }
\vspace{-0.2em}
\begin{center}
\setlength{\tabcolsep}{3pt}
\scalebox{0.86}{
\begin{tabular}[t]{c  l | cccc | cccc | cccc | cccc | cccc}
\toprule
&  &  \multicolumn{4}{c|}{CUB} & \multicolumn{4}{c|}{AwA2}  &  \multicolumn{4}{c}{SUN} &  \multicolumn{4}{c}{APY}&  \multicolumn{4}{c}{DeepFashion}  \\ 
&  &  Acc$_{\mathcal{C}}$ &  Acc$_{\mathcal{Y}^{u}}$ & Acc$_{\mathcal{Y}^{s}}$ & Acc$_{\mathcal{H}}$  & Acc$_{\mathcal{C}}$ & Acc$_{\mathcal{Y}^{u}}$ & Acc$_{\mathcal{Y}^{s}}$ & Acc$_{\mathcal{H}}$  & Acc$_{\mathcal{C}}$ & Acc$_{\mathcal{Y}^{u}}$ & Acc$_{\mathcal{Y}^{s}}$ & Acc$_{\mathcal{H}}$ & Acc$_{\mathcal{C}}$ & Acc$_{\mathcal{Y}^{u}}$ & Acc$_{\mathcal{Y}^{s}}$ & Acc$_{\mathcal{H}}$& Acc$_{\mathcal{C}}$ & Acc$_{\mathcal{Y}^{u}}$ & Acc$_{\mathcal{Y}^{s}}$ & Acc$_{\mathcal{H}}$\\

\hline 
\parbox[t]{2mm}{\multirow{4}{*}{\rotatebox[origin=c]{90}{\textit{\textbf{centralized}}}}} 
&    {APN  }  
    & 72.0  & 65.3  & 69.3  & 67.2
    & 68.4  & 56.5  & 78.0  & 65.5
    & 61.6  & 41.9  & 34.0  & 37.6
    & 38.7  & 18.9  & 43.7  & 26.3
    & 35.2  & 25.6  & 34.3  & 29.3
    \\
&    {GEM  }  
    & 77.8  & 64.8    & 77.1  & 70.4
    & 67.3  & 64.8    & 77.5  & 70.6
    & 62.8  & 38.1    & 35.7  & 36.9
    & 39.4  & 19.8    & 45.2  & 27.5
    & 33.1  & 24.8    & 33.1  & 28.3
    \\
&    {MSDN  }  
    & 76.1  & 68.7  & 67.5  & 68.1
    & 70.1  & 62.0  & 74.5  & 67.7
    & 65.8  & 52.2  & 34.2  & 41.3
    & 37.2  & 18.2  & 43.8  & 25.7
    & 28.7  & 21.8  & 29.2  & 25.0
    \\
&    {SVIP*  }
    & 79.8 & 72.1 & 78.1 & 75.0
    & 69.8 & 65.4 & 87.7 & 76.9
    & 71.6 & 53.7 & 48.0 & 50.7
    & 41.3 & 23.9 & 37.1 & 29.1
    & 36.2 & 29.8 & 30.1 & 30.0
\\
    \cline{2-22}
&  {DistZSL}  
    & 73.9          & 62.4              & 70.1             & 66.1
    & 68.5          & 61.0              & 71.2             & 65.7
    & 61.4          & 42.2              & 29.8             & 35.0
    & 38.3          & 18.5              & 44.9             & 26.2
    & 34.6          & 24.2              & 34.7             & 28.5
    \\
&  {DistZSL*}  
    & 82.4          & 71.3              & 76.8             & 73.9
    & 66.7          & 64.3              & 84.7             & 73.1
    & 72.4          & 53.8              & 47.4             & 50.4
    & 40.7          & 21.8              & 45.4             & 29.5
    & 35.9          & 26.2              & 33.8             & 29.5
    \\
\bottomrule
\parbox[t]{2mm}{\multirow{21}{*}{\rotatebox[origin=c]{90}{\textit{\textbf{i.i.d.}}}}} &
 {FedAvg}  
    & --            & --                & 41.1             & --
    & --            & --                & 90.1             & --
    & --            & --                & 0.5              & --
    & --            & --                & 63.2             & --
    & --            & --                & 38.4             & --
    \\
&  {+ APN  }          
    & 68.2          & 59.1              & 60.7            & 59.9
    & 54.5          & 38.9              & 76.2            & 51.5
    & 20.5          & 12.2              & 6.1             & 8.1
    & 34.2          & 8.9               & 47.5            & 15.0
    & 26.1          & 16.3              & {25.8}   & 20.0
    \\
&  {+ GEM  }          
    & 67.4          & 38.7              & 64.1            & 48.2
    & 61.3          & 28.6              & 78.5            & 42.0
    & 61.0          & 32.9              & 31.6            & 32.2
    & 35.1          & 11.1              & 45.8            & 17.9
    & 25.5          & 19.7              & 23.2            & 21.3
    \\
&  {+ MSDN  }          
    & 68.3          & 23.4              & 49.4            & 31.7
    & 57.0          & 17.9              & 70.6            & 28.5
    & 58.4          & 28.2              & 33.4            & 30.6
    & 30.1          & 9.8               & 49.8            & 16.4
    & 20.3          & 6.5               & 24.1            & 10.3
    \\
&  {+ SVIP*  }
    & 79.4 & 58.9 & 71.7 & 64.7 
    & 63.2 & 57.2 & 85.1 & 68.4
    & 68.1 & 50.1 & 46.4 & 48.2
    & 38.2 & 12.6 & 48.2 & 20.0
    & 25.1 & 26.5 & 32.3 & 29.1
    \\
&    {FedProx}        
    & --            & --                & 41.2            & --
    & --            & --                & 89.9            & --
    & --            & --                & 0.7             & --
    & --            & --                & 62.8            & --
    & --            & --                & 38.9            & --
    \\
&    {+ APN}          
    & 67.1          & 58.1             & 62.2            & 60.1
    & 56.6          & 42.2             & 76.2            & 54.3
    & 52.7          & 33.0             & 26.2            & 29.2
    & 34.5          & 9.1              & 49.4            & 15.4
    & 25.0          & 17.4             & 22.1            & 19.5
    \\
&    {+ GEM}          
    & 68.2          & 37.4              & {69.7}   & 48.7
    & 58.9          & 29.2              & 78.5            & 42.6
    & 62.4          & 29.8              & {38.4}   & 33.5
    & 33.1          & 11.3              & 46.7            & 18.2
    & 26.3          & 15.9              & 24.7            & 19.3
    \\
&    {+ MSDN}          
    & 68.8          & 23.9              & 50.7            & 32.5
    & 58.4          & 17.3              & 72.3            & 27.9
    & 57.1          & 29.1              & 32.4            & 30.7
    & 34.2          & 10.4              & 45.2            & 16.9
    & 21.5          & 8.7               & 22.9            & 12.6
    \\
&    {+ SVIP*}
    & 78.3  & 57.8  & 70.1  & 63.4
    & 62.1  & 56.1  & 85.0  & 67.6
    & 66.4  & 47.5  & 43.4  & 45.4
    & 37.4  & 11.2  & 47.9  & 18.2
    & 25.3  & 25.8  & 31.8  & 28.5
    \\

&  {FedNova}         
    & --            & --                & 41.7            & --
    & --            & --                & {90.9}   & --
    & --            & --                & 0.5             & --
    & --            & --                & 63.8            & --
    & --            & --                & 37.9            & --
    \\
&  {+ APN  }          
    & 67.9          & 58.0              & 60.8            & 59.4
    & 54.5          & 40.2              & 75.7            & 52.5
    & 37.4          & 21.7              & 13.8            & 16.9
    & 29.7          & 9.8               & 48.7            & 16.3
    & 25.7          & 17.8              & 21.9            & 19.6
    \\
&  {+ GEM  }          
    & 67.7          & 37.3              & 69.6            & 48.6
    & 58.6          & 26.6              & 78.0            & 39.6
    & 61.8          & 30.1              & 36.7            & 33.1
    & 32.8          & 6.9               & 49.6            & 12.1
    & 24.8          & 18.2              & 23.9            & 20.7
    \\
&  {+ MSDN  }          
    & 68.3          & 26.6              & 36.1            & 30.7
    & 58.4          & 19.8              & 79.1            & 31.7
    & 59.4          & 28.5              & 33.9            & 31.0
    & 29.8          & 8.8               & 46.9            & 14.8
    & 19.9          & 5.9               & 22.6            & 9.4
    \\
&  {+ SVIP  }
    & 79.1 & 58.7 & 71.4  & 64.4
    & 61.8 & 56.7 & 84.4  & 67.8
    & 67.3 & 47.6 & 43.8  & 45.6
    & 36.1 & 10.4 & 48.2  & 17.1
    & 24.4 & 22.1 & 29.4  & 25.2
    \\
&    {Scaffold}        
    & --            & --               & 45.1            & --
    & --            & --               & 90.2            & --
    & --            & --               & 1.1             & --
    & --            & --               & 63.4            & --
    & --            & --               & 37.5            & --
    \\
&    {+ APN  }          
    & 69.7          & 60.8              & 60.3            & 60.5
    & 55.1          & 37.4              & 72.4            & 49.3
    & 35.2          & 18.4              & 12.1            & 14.6
    & 34.8          & 9.5               & 48.3            & 15.9
    & 26.3          & 16.8              & 25.7            & 20.3
    \\
&    {+ GEM  }          
    & 68.2          & 37.7              & 67.3            & 48.3
    & 50.2          & 37.7              & 66.3            & 48.1
    & 61.3          & 30.1              & 36.0            & 32.8
    & 35.1          & 10.6              & 46.9            & 17.3
    & 26.2          & 18.4              & 24.4            & 21.0
    \\
&    {+ MSDN  }          
    & 70.2          & 27.4              & 45.2            & 34.1
    & 58.5          & 19.2              & 70.8            & 30.2
    & 58.5          & 28.2              & 34.8            & 31.2
    & 31.3          & 9.9               & {50.1}   & 16.5
    & 21.1          & 9.4               & 22.8            & 13.3
    \\
&    {+ SVIP  } 
    & 80.2  & 58.6  & 70.7  & 64.1
    & 61.1  & 57.1  & 83.7  & 67.9
    & 67.1  & 48.4  & 43.5  & 45.8
    & 38.5  & 10.7  & 47.8  & 17.5
    & 24.3  & 24.3  & 30.4  & 27.0
    \\
&  {MOON}         
    & --            & --                 & 41.0           & --
    & --            & --                 & 90.3           & --
    & --            & --                 & 0.6            & --
    & --            & --                 & 64.3           & --
    & --            & --                 & 39.5           & --
    
    \\
&  {+ APN  }          
    & 67.4          & 57.7              & 62.7            & 60.1
    & 55.3          & 37.5              & 85.4            & 52.1
    & 3.5           & 1.3               & 0.2             & 0.3
    & 33.8          & 9.5               & 48.6            & 15.9
    & 24.6          & 16.6              & 24.1            & 19.7
    \\
&  {+ GEM  }          
    & 66.4          & 34.3              & 66.2            & 45.2
    & 59.8          & 30.1              & 78.3            & 43.4
    & 59.6          & 25.2              & 35.0            & 29.3
    & 33.9          & 10.1              & 47.8            & 16.7
    & 23.8          & 17.4              & 24.9            & 20.5
    \\
&  {+ MSDN  }          
    & 68.4          & 24.7              & 49.6            & 33.0
    & 57.4          & 17.7              & 80.9            & 29.0
    & 59.2          & 28.8              & 31.7            & 30.2
    & 26.4          & 10.4              & 46.9            & 17.0
    & 25.8          & 8.9              & 25.6             & 13.2
    \\
&  {+ SVIP*  }
    & 79.6  & 57.1  & 69.3  & 62.6
    & 64.1  & 57.7  & 85.1  & 68.7
    & 69.1  & 51.4  & 45.2  & 48.1
    & 36.7  & 10.0  & 48.3  & 16.6
    & 25.3  & 26.1  & 32.2  & 28.8
    \\
&    {FedGloss}        
    & --            & --               & 40.8            & --
    & --            & --               & 90.2            & --
    & --            & --               & 1.2             & --
    & --            & --               & 63.3            & --
    & --            & --               & 38.6            & --
    \\
&    {+ APN  }          
    & 67.0  & 58.4  & 59.8 & 59.1
    & 54.3  & 38.4  & 73.8 & 50.5
    & 32.4  & 17.6  & 21.4 & 19.3
    & 31.3  & 7.4   & 45.7 & 12.7
    & 25.5  & 15.8  & 24.3 & 19.1
    \\
&    {+ GEM  }          
    & 67.2  & 57.4  & 60.1  & 58.7
    & 55.4  & 39.4  & 74.3  & 51.5
    & 33.5  & 17.7  & 22.2  & 19.7
    & 31.4  & 8.1   & 46.5  & 13.8
    & 26.0  & 14.7  & 22.8  & 17.9
    \\
&    {+ MSDN  }          
    & 65.4  & 56.1  & 58.4  & 57.2
    & 51.8  & 34.2  & 71.9  & 46.4
    & 24.8  & 10.4  & 17.3  & 13.0
    & 24.1  & 7.5   & 37.4  & 12.5
    & 19.4  & 9.4   & 30.1  & 14.3
    \\
&    {+ SVIP  } 
    & 78.4 & 56.6 & 68.4 & 61.9
    & 63.7 & 56.1 & 85.4 & 67.7
    & 68.2 & 49.6 & 44.9 & 47.1
    & 38.7 & 12.7 & 45.0 & 19.9
    & 26.8 & 27.3 & 32.4 & 29.6
    \\
    \cline{2-22}
&   {DistZSL}
    & {71.0} & {61.6} & 62.1 & {61.8}
    & {59.7} & {52.7} & 74.5 & {61.8}
    & {63.3} & {43.3} & 29.6 & {35.2}
    & {36.1} & {11.8} & 49.9 & {19.1} 
    & {27.2} & {21.5} & 24.1 & {22.7}
    \\
&   {DistZSL*}
    & {81.2} & {57.5} & 69.6 & {63.0}
    & 65.8 & 59.4 & 85.5 & 70.1 
    & 70.8 & 53.4 & 45.8 & 49.3
    & 40.3 & 14.2 & 56.6 & 22.7
    & 28.4 & 29.3 & 32.4 & 30.8
    \\
\bottomrule
\end{tabular}}
\end{center}
\label{performance2}
\end {table*}

\begin {table*}[!h]
\caption{Performance comparisons (\%) on five datasets among FL baselines, ZSL baselines, and the proposed DistZSL in \textit{Non-i.i.d.} and \textit{p.c.c.d.} settings. }
\vspace{-0.2em}
\begin{center}
\setlength{\tabcolsep}{3pt}
\scalebox{0.86}{
\begin{tabular}[t]{c  l | cccc | cccc | cccc | cccc | cccc}
\toprule
&  &  \multicolumn{4}{c|}{CUB} & \multicolumn{4}{c|}{AwA2}  &  \multicolumn{4}{c}{SUN} &  \multicolumn{4}{c}{APY}&  \multicolumn{4}{c}{DeepFashion}  \\ 
&  &  Acc$_{\mathcal{C}}$ &  Acc$_{\mathcal{Y}^{u}}$ & Acc$_{\mathcal{Y}^{s}}$ & Acc$_{\mathcal{H}}$  & Acc$_{\mathcal{C}}$ & Acc$_{\mathcal{Y}^{u}}$ & Acc$_{\mathcal{Y}^{s}}$ & Acc$_{\mathcal{H}}$  & Acc$_{\mathcal{C}}$ & Acc$_{\mathcal{Y}^{u}}$ & Acc$_{\mathcal{Y}^{s}}$ & Acc$_{\mathcal{H}}$ & Acc$_{\mathcal{C}}$ & Acc$_{\mathcal{Y}^{u}}$ & Acc$_{\mathcal{Y}^{s}}$ & Acc$_{\mathcal{H}}$& Acc$_{\mathcal{C}}$ & Acc$_{\mathcal{Y}^{u}}$ & Acc$_{\mathcal{Y}^{s}}$ & Acc$_{\mathcal{H}}$\\

 \hline 
\parbox[t]{2mm}{\multirow{21}{*}{\rotatebox[origin=c]{90}{\textit{\textbf{Non-i.i.d.}}}}} 
&    {FedAvg}  
    & --            & --                & 6.4             & --
    & --            & --                & 18.4            & --
    & --            & --                & 1.9             & --
    & --            & --                & 21.1            & --
    & --            & --                & 12.8            & --
    \\
&    {+ APN  }          
    & 65.0          & 54.9              & 60.9            & 57.7
    & 53.7          & 41.9              & 76.2            & 54.1
    & 35.3          & 20.8              & 14.2            & 16.8
    & 28.1          & 11.4              & 36.9            & 17.4
    & 25.1          & 12.8              & 20.1            & 15.6
    \\
&    {+ GEM  }          
    & 67.2          & 37.9              & 62.8            & 47.3
    & 57.4          & 29.9              & 60.0            & 39.9
    & 60.2          & 30.8              & 33.1            & 31.9
    & 28.5          & 10.9              & 37.4            & 16.9
    & 24.7          & 6.3               & 17.2            & 9.2
    \\
&    {+ MSDN  }          
    & 64.8          & 25.3              & 40.5            & 31.2
    & 56.9          & 18.9              & 67.9            & 29.6
    & 57.6          & 29.4              & 32.3            & 30.8
    & 25.6          & 7.3               & 35.1            & 12.1
    & 20.1          & 10.0              & 18.8            & 13.0
    \\
&    {+ SVIP*  }
    & 75.2  & 52.4  & 68.9  & 59.5
    & 59.7  & 51.2  & 70.8  & 59.4
    & 66.1  & 48.6  & 43.8  & 46.1
    & 34.0  & 12.1  & 41.3  & 18.7
    & 25.5  & 15.0  & 26.6  & 19.2
    \\
&  {FedProx}        
    & --             & --               & 6.6             & --
    & --             & --               & 21.0            & --
    & --             & --               & 2.3             & --
    & --             & --               & 20.7            & --
    & --             & --               & 11.9            & --
    \\
&  {+ APN  }          
    & 64.7          & 54.9              & 59.7            & 57.2
    & 56.5          & 43.7              & 81.6            & 56.9
    & 50.8          & 34.4              & 24.0            & 28.3
    & 27.9          & 10.9              & 35.7            & 16.7
    & 25.3          & 13.8              & 23.5            & 17.4
    \\
&  {+ GEM  }          
    & 69.1          & 36.7              & 66.3            & 47.3
    & 58.7          & 35.2              & 61.2            & 44.7
    & 60.6          & 28.1              & {37.7}   & 32.2
    & 28.4          & 11.3              & 34.1            & 17.0
    & 24.9          & 12.4              & 25.1            & 16.6
    \\
&  {+ MSDN  }          
    & 67.5          & 26.1              & 38.7            & 31.2
    & 58.1          & 22.3              & 51.1            & 31.1
    & 58.1          & 30.4              & 33.0            & 31.7
    & 27.5          & 9.6               & 36.8            & 15.2
    & 24.6          & 9.6               & 24.7            & 13.8
    \\
& {+ SVIP*}
    & 74.6  & 52.1  & 67.4  & 58.8
    & 58.4  & 48.6  & 69.3  & 57.1
    & 64.8  & 45.9  & 44.0  & 44.9
    & 32.8  & 11.7  & 39.6  & 18.1
    & 24.9  & 13.9  & 26.4  & 18.2
    \\
&    {FedNova}         
    & --            & --                & 6.9             & --
    & --            & --                & 19.0            & --
    & --            & --                & 2.4             & --
    & --            & --                & 19.6            & --
    & --            & --                & 12.6            & --
    
    \\
&    {+ APN  }          
    & 66.8          & 55.9              & 59.8          & 57.8
    & 36.0          & 30.1              & 36.4          & 33.0
    & 37.6          & 21.3              & 14.7          & 17.3
    & 29.3          & 11.2              & 37.1          & 17.2
    & 23.7          & 11.8              & 23.1          & 15.6
    \\
&    {+ GEM  }          
    & 67.4          & 39.0              & 61.5          & 47.7
    & 58.4          & 30.1              & 62.2          & 40.5
    & 60.9          & 32.0              & 32.3          & 32.2
    & 27.6          & 10.7              & 37.2          & 16.6
    & 24.8          & 12.7              & 23.5          & 16.5
    \\
&    {+ MSDN  }          
    & 65.6          & 34.7              & 38.7          & 30.2
    & 57.3          & 25.1              & 49.3          & 33.2
    & 58.6          & 27.6              & 34.7          & 30.8
    & 28.1          & 9.7               & 34.9          & 15.2
    & 23.4          & 10.8              & 16.9          & 13.2
    \\
&    {+ SVIP*}
    & 74.1  & 50.9  & 66.4  & 57.6
    & 59.4  & 48.5  & 68.9  & 56.9
    & 64.2  & 44.3  & 40.1  & 42.1
    & 32.1  & 11.3  & 38.6  & 17.5
    & 24.7  & 13.8  & 26.9  & 18.2
    \\
&   Scaffold         
    & --            & --                & 7.7           & --
    & --            & --                & 19.7          & --
    & --            & --                & 2.6           & --
    & --            & --                & 23.7          & --
    & --            & --                & 12.9          & --
    \\
&  {+ APN  }          
    & 67.9          & 56.4              & 60.9          & 58.6
    & 54.4          & 37.4              & 82.6          & 51.5
    & 33.7          & 19.4              & 12.9          & 15.5
    & 29.6          & 12.8              & 38.1          & 19.2
    & 25.7          & 13.2              & 21.8          & 16.4
    \\
&  {+ GEM  }          
    & 68.4          & 35.3              & {66.9} & 46.2
    & 57.7          & 34.0              & 61.8          & 43.8
    & 61.3          & 30.0              & 36.8          & 33.1
    & 29.8          & 11.8              & 37.4          & 17.9
    & 25.3          & 13.0              & 22.6          & 16.5
    
    \\
&  {+ MSDN  }          
    & 68.9          & 26.3              & 49.5          & 34.4
    & 54.7          & 30.5              & 43.4          & 35.8
    & 59.0          & 29.0              & 35.0          & 31.7
    & 27.9          & 10.5              & 34.6          & 16.1
    & 23.2          & 12.1              & 23.9          & 16.1
    \\
&  {+ SVIP*}
    & 75.4  & 52.8  & 67.1  & 59.1
    & 60.8  & 50.9  & 70.2  & 59.0
    & 66.8  & 49.2  & 45.3  & 47.2
    & 34.6  & 13.2  & 41.6  & 20.0
    & 27.1  & 16.8  & 26.9  & 20.7
    \\

&    {MOON}          
    & --            & --                & 7.3           & --
    & --            & --                & 20.9          & --
    & --            & --                & 2.4           & --
    & --            & --                & 22.0          & --
    & --            & --                & 13.1          & --
    \\
&    {+ APN}          
    & 66.2          & 58.1              & 58.3          & 58.2
    & 54.9          & 41.6              & 78.5          & 54.4
    & 3.9           & 1.4               & 0.2           & 0.3
    & 30.0          & 12.3              & 35.1          & 18.2
    & 24.9          & 11.4              & 24.7          & 15.6
    \\
&    {+ GEM}          
    & 66.0          & 33.2              & 62.8          & 43.5
    & 58.1          & 28.9              & 62.5          & 39.5
    & 57.1          & 28.3              & 30.2          & 29.2
    & 30.3          & 12.7              & 34.7          & 18.6
    & 25.3          & 12.5              & 25.0          & 16.7
    \\
&    {+ MSDN}          
    & 67.6          & 27.9              & 36.5          & 31.6
    & 55.8          & 23.5              & 48.4          & 31.6
    & 58.2          & 29.8              & 31.6          & 30.7
    & 28.6          & 10.8              & 34.0          & 16.4
    & 23.8          & 11.4              & 25.3          & 15.7
    \\
    
&    {+ SVIP*}
    & 74.6  & 51.9  & 66.5  & 58.3
    & 59.0  & 49.4  & 68.9  & 57.5
    & 65.4  & 47.4  & 43.1  & 45.1
    & 34.1  & 12.4  & 39.8  & 18.9
    & 25.0  & 17.4  & 28.7  & 21.7
    \\
&   FedGloss         
    & --            & --                & 7.7           & --
    & --            & --                & 19.7          & --
    & --            & --                & 2.6           & --
    & --            & --                & 23.7          & --
    & --            & --                & 12.9          & --
    \\
&  {+ APN  }          
    & 67.3          & 55.7              & 61.1          & 58.3
    & 53.4          & 34.6              & 81.0          & 48.5
    & 33.9          & 20.0              & 12.1          & 15.1
    & 28.4          & 12.3              & 37.4          & 18.5
    & 25.2          & 12.9              & 21.4          & 16.1
    \\
&  {+ GEM  }          
    & 68.0          & 34.8              & 65.4          & 45.4
    & 56.9          & 33.8              & 60.6          & 43.4
    & 60.3          & 28.1              & 35.0          & 31.2
    & 28.3          & 10.2              & 36.4          & 15.9
    & 23.8          & 11.4              & 20.6          & 14.7
    \\
&  {+ MSDN  }          
    & 67.6          & 21.3              & 55.6          & 30.8
    & 53.1          & 28.6              & 51.2          & 36.7
    & 55.8          & 21.3              & 32.4          & 25.7
    & 27.0          & 9.4               & 35.7          & 14.9
    & 21.8          & 10.9              & 21.8          & 14.5
    \\
&  {+ SVIP*}
    & 74.8  & 52.7  & 66.4  & 58.8
    & 59.9  & 50.6  & 69.7  & 58.6
    & 66.1  & 48.7  & 44.7  & 46.6
    & 33.4  & 12.9  & 40.7  & 19.6
    & 25.6  & 15.4  & 27.9  & 19.8
    \\
\cline{2-22}
&  {DistZSL}
    & {71.4} & {58.9}     & 62.0   & {60.4}
    & {58.7} & {51.6}     & {70.0}   & {59.5}
    & {61.9} & {39.5}     & 30.7            & {34.5}
    & {34.8} & {13.4}     & {39.4}   & {20.0}
    & {26.5} & {17.1}     & {28.1}   & {21.2} 
    \\
    & {DistZSL*}
    & 80.3 & 54.3 & 69.7 & 61.0
    & 63.4 & 53.8 & 73.4 & 62.1
    & 68.7 & 52.4 & 45.1 & 48.5
    & 36.7 & 14.5 & 50.7 & 22.6
    & 27.3 & 17.6 & 29.4 & 22.0
    
    \\
\bottomrule
\parbox[t]{2mm}{\multirow{21}{*}{\rotatebox[origin=c]{90}{\textit{\textbf{p.c.c.d.}}}}}&
  {FedAvg}   
    & --            & --                & 5.2           & --
    & --            & --                & 8.8           & --
    & --            & --                & 0.3           & --
    & --            & --                & 9.5           & --
    & --            & --                & 3.9           & --
    \\
&  {+ APN}          
    & 50.9          & 41.9              & 50.4          & 45.8
    & 33.1          & 24.9              & 29.9          & 27.2
    & 33.0          & 18.8              & 13.7          & 15.9
    & 17.1          & 10.7              & 24.5          & 14.9
    & 15.8          & 10.6              & 8.0           & 9.1
    \\
&  {+ GEM}          
    & 51.8          & 30.8              & 50.5          & 38.2
    & 43.0          & 19.5              & 33.9          & 24.7
    & 57.3          & 31.0              & 32.6          & 31.8
    & 17.2          & 11.9              & 26.1          & 16.3
    & 16.8          & 12.6              & 10.4          & 11.4
    \\
&  {+ MSDN}          
    & 49.7          & 19.9              & 20.1          & 20.0
    & 38.4          & 20.1              & 44.5          & 27.7
    & 53.8          & 25.7              & 27.3          & 26.5
    & 15.7          & 10.1              & 23.8          & 14.2
    & 16.1          & 9.9               & 7.9           & 8.8 
    \\
&  {+ SVIP*}
    & 73.8          & 47.8              & 63.0          & 54.4
    & 54.9          & 42.4              & 74.9          & 54.1
    & 63.8          & 44.1              & 40.4          & 42.2
    & 29.8          & 14.0              & 35.3          & 20.0
    & 21.1          & 13.8              & 23.9          & 17.5
    \\
&    {FedProx}         
    & --            & --                & 6.2           & --
    & --            & --                & 8.8           & --
    & --            & --                & 0.2           & --
    & --            & --                & 9.8           & --
    & --            & --                & 4.2           & --
    \\
&    {+ APN  }          
    & 52.3          & 41.2              & 50.9          & 45.5
    & 44.2          & 37.6              & 51.0          & 43.3
    & 44.5          & 28.3              & 20.9          & 24.0
    & 14.3          & 9.0               & 23.9          & 13.1
    & 16.2          & 11.3              & 8.8           & 9.9
    \\
&    {+ GEM  }          
    & 50.3          & 30.2              & 53.2          & 38.5
    & 53.0          & 32.8              & 52.6          & 40.4
    & 56.9          & 39.9              & 32.9          & 31.4
    & 14.6          & 11.3              & 10.1          & 10.7
    & 16.8          & 11.7              & 9.2           & 10.3
    \\
&    {+ MSDN  }          
    & 53.0          & 19.6              & 56.8          & 29.1
    & 47.1          & 12.6              & 31.0          & 17.9
    & 54.9          & 22.4              & 28.5          & 25.0
    & 13.4          & 9.4               & 29.4          & 14.2
    & 15.9          & 10.3              & 7.5           & 8.7
    \\
&    {+ SVIP*}
    & 72.1          & 49.1          & 64.3      & 55.7
    & 54.0          & 41.2          & 75.4      & 53.3
    & 63.4          & 46.0          & 42.4      & 44.1
    & 28.3          & 13.4          & 33.2      & 19.1
    & 20.6          & 10.9          & 20.8      & 14.3
    \\
&   {FedNova}        
    & --            & --                & 6.1           & --
    & --            & --                & 8.8           & --
    & --            & --                & 0.3           & --
    & --            & --                & 9.0           & --
    & --            & --                & 4.1           & --
    \\
&  {+ APN  }          
    & 51.8          & 39.1              & 53.5          & 45.2
    & 35.1          & 25.9              & 31.4          & 28.4
    & 38.3          & 23.2              & 15.9          & 18.9
    & 16.4          & 10.5              & 25.1          & 14.8
    & 15.6          & 11.2              & 9.9           & 10.5
    \\
&  {+ GEM  }
    & 52.0          & 31.7              & 51.9          & 39.3
    & 43.0          & 19.8              & 33.9          & 25.0
    & 57.4          & 27.9              & 34.6          & 30.9
    & 17.3          & 11.6              & 28.1          & 16.4
    & 16.3          & 12.4              & 9.4           & 10.7
    \\
&  {+ MSDN  }          
    & 52.5          & 19.8              & 24.6          & 21.9
    & 41.9          & 26.1              & 58.3          & 36.0
    & 53.4          & 24.2              & 30.0          & 26.8
    & 15.7          & 8.8               & 24.0          & 12.9
    & 15.7          & 10.8              & 8.5           & 9.5
    \\
&  {+ SVIP*  }
    & 70.8   & 48.5  & 65.4  & 55.7
    & 53.8   & 41.1  & 74.8  & 53.1
    & 64.0   & 46.2  & 42.8  & 44.4
    & 28.0   & 13.3  & 34.5  & 19.2
    & 21.5   & 14.0  & 24.8  & 17.9 
    \\

&    {Scaffold}        
    & --            & --                & 6.0           & --
    & --            & --                & 9.0           & --
    & --            & --                & 0.3           & --
    & --            & --                & 10.3          & --
    & --            & --                & 4.8           & --
    \\
&    {+ APN  }          
    & 52.8          & 44.8              & 48.8          & 46.7
    & 40.3          & 36.1              & 40.6          & 38.2
    & 36.4          & 16.0              & 11.8          & 13.6
    & 18.9          & 12.3              & 28.4          & 17.2
    & 17.6          & 12.7              & 10.0          & 11.2
    \\
&    {+ GEM  }          
    & 59.9          & 30.8              & 56.0          & 39.7
    & 49.4          & 24.3              & 38.9          & 29.9
    & 58.5          & 29.4              & {33.5} & 31.3
    & 19.0          & 12.8              & 27.4          & 17.4
    & 18.2          & 12.4              & 9.7           & 10.9
    \\
&    {+ MSDN  }          
    & 57.0          & 18.8              & 35.3          & 24.5
    & 45.2          & 10.8              & 54.3          & 18.0
    & 57.2          & 26.7              & 33.2          & 29.6
    & 17.1          & 8.8               & 21.4          & 12.5
    & 16.4          & 11.3              & 7.7           & 9.2
    \\
&    {+ SVIP*}
    & 73.1  & 49.2  & 64.7  & 55.9
    & 55.6  & 41.9  & 75.8  & 54.0
    & 63.6  & 47.8  & 44.1  & 45.9
    & 30.1  & 13.5  & 36.1  & 19.7
    & 22.1  & 15.1  & 25.0  & 18.8
    \\

&  {MOON}         
    & --            & --                & 6.1           & --
    & --            & --                & 8.9           & --
    & --            & --                & 0.2           & --
    & --            & --                & 8.1           & --
    & --            & --                & 3.7           & --
    \\
&  {+ APN  }          
    & 51.6          & 40.3              & 49.8          & 44.6
    & 34.8          & 27.1              & 32.3          & 29.5
    & 4.0           & 0.9               & 0.2           & 0.4
    & 17.2          & 10.1              & 14.3          & 11.8
    & 15.8          & 10.6              & 9.3           & 9.9
    \\
&  {+ GEM  }
    & 43.6          & 31.6              & 41.4          & 35.9
    & 44.9          & 28.3              & 32.9          & 30.4
    & 54.2          & 26.9              & 29.8          & 28.3
    & 17.6          & 10.4              & 13.4          & 11.7
    & 16.4          & 10.2              & 9.1           & 9.6
    \\
&  {+ MSDN  }          
    & 50.2          & 15.6              & 39.0          & 22.3
    & 33.2          & 25.4              & 60.6          & 33.8
    & 54.4          & 26.9              & 25.9          & 26.4
    & 17.0          & 9.4               & 12.1          & 10.6
    & 17.3          & 13.1              & 11.4          & 12.2
    \\
& {+ SVIP*}
    & 71.6  & 49.7  & 63.8  & 55.9
    & 54.2  & 43.1  & 73.9  & 54.4
    & 64.9  & 46.9  & 43.8  & 45.3
    & 28.7  & 12.8  & 35.1  & 18.8
    & 20.7  & 13.2  & 24.0  & 17.0
    \\
&    {FedGloss}        
    & --            & --                & 6.0           & --
    & --            & --                & 9.0           & --
    & --            & --                & 0.3           & --
    & --            & --                & 10.3          & --
    & --            & --                & 4.8           & --
    \\
&    {+ APN  }          
    & 51.0          & 40.1              & 48.6          & 43.9
    & 33.8          & 30.2              & 37.4          & 33.4
    & 33.8          & 18.3              & 13.7          & 15.7
    & 17.0          & 9.5               & 25.3          & 13.8
    & 15.9          & 11.2              & 7.0           & 8.6
    \\
&    {+ GEM  }          
    & 55.4          & 30.1              & 54.8          & 38.9
    & 44.2          & 26.1              & 31.9          & 28.7
    & 54.9          & 27.4              & 28.4          & 27.9
    & 17.1          & 11.4              & 26.8          & 16.0
    & 17.1          & 10.4              & 10.3          & 10.3
    \\
&    {+ MSDN  }          
    & 53.4          & 16.4              & 41.4          & 23.5
    & 43.8          & 11.4              & 53.2          & 18.8
    & 55.1          & 25.8              & 29.1          & 27.4
    & 16.4          & 8.3               & 18.9          & 11.5
    & 15.1          & 9.3               & 8.3           & 8.8
    \\
&    {+ SVIP*}
    & 71.3  & 48.4  & 64.2  & 55.2
    & 54.8  & 42.1  & 75.1  & 54.0
    & 65.4  & 47.2  & 43.7  & 45.4
    & 29.4  & 13.1  & 34.8  & 19.0
    & 21.3  & 14.8  & 24.3  & 18.4
\\
\cline{2-22}
&  {DistZSL}  
    & {71.6} & {57.5}     & {58.0}   & {57.8}
    & {57.2} & {45.5}     & {62.3}   & {52.6}
    & {60.9} & {39.9}     & 27.2     & {32.3}
    & {19.2} & {16.1}     & {29.6}   & {20.8}
    & {23.2} & {16.2}     & {15.6}   & {15.9}
    \\
    & DistZSL*
    & 79.5          & 50.7              & 69.6              & 58.7
    & 57.3          & 46.0              & 79.7              & 58.3
    & 67.8          & 51.7              & 44.2              & 47.7
    & 31.3          & 15.4              & 45.9              & 23.0
    & 23.4          & 16.6              & 26.8              & 20.5
    \\
\hline
\bottomrule
\end{tabular}}
\end{center}
\label{performance3}
\end {table*} 
\end{document}

%% file: math.tex
\usepackage{amsmath,amsfonts,bm}

\def\eqref#1{equation~\ref{#1}}

\def\1{\bm{1}}

\def\va{{\bm{a}}}

\def\vv{{\bm{v}}}

\def\vx{{\bm{x}}}

\def\vz{{\bm{z}}}

\DeclareMathAlphabet{\mathsfit}{\encodingdefault}{\sfdefault}{m}{sl}
\SetMathAlphabet{\mathsfit}{bold}{\encodingdefault}{\sfdefault}{bx}{n}

\newcommand{\softmax}{\mathrm{softmax}}

\DeclareMathOperator*{\argmin}{arg\,min}